%% file: main.tex
\documentclass[twoside]{article}

\usepackage[accepted]{aistats2022}
\usepackage{alias}

\usepackage{selectp}

\bptrue
\bpfalse

\newcommand{\titlename}{Jointly Efficient and Optimal Algorithms for Logistic Bandits}

\begin{document}

\doparttoc
\faketableofcontents

\twocolumn[

\aistatstitle{\titlename}

\aistatsauthor{ Louis Faury \And Marc Abeille \And  Kwang-Sung Jun \And Cl\'ement Calauz\`enes}

\aistatsaddress{ Criteo AI Lab \And  Criteo AI Lab \And University of Arizona \And Criteo AI Lab } ]

\begin{abstract}
  	Logistic Bandits have recently undergone careful scrutiny by virtue of their combined theoretical and practical relevance. This research effort delivered statistically efficient algorithms,  improving the regret of previous strategies by exponentially large factors. Such algorithms are however strikingly costly as they require $\bigomega{t}$ operations at each round. On the other hand, a different line of research focused on computational efficiency ($\bigo{1}$ per-round cost), but at the cost of letting go of the aforementioned exponential improvements. Obtaining the best of both world is unfortunately not a matter of marrying both approaches. 
	Instead we introduce a new \emph{learning} procedure for Logistic Bandits. It yields confidence sets which sufficient statistics can be easily maintained online without sacrificing statistical tightness. Combined with efficient \emph{planning} mechanisms we design fast algorithms which regret performance still match the problem-dependent lower-bound of \cite{abeille2020instance}. To the best of our knowledge, those are the first Logistic Bandit algorithms that simultaneously enjoy statistical and computational efficiency. 

\end{abstract}

\input{content/intro}

\input{content/preliminaries}

\input{content/results}

\input{content/dilute}

\bibliography{bib.bib}
\bibliographystyle{apalike}

\addcontentsline{toc}{section}{Appendix}
\appendix
\part{}

\appendix 
\onecolumn
\aistatstitle{\titlename \\ Supplementary Material}

\section*{\uppercase{Organization of the Appendix}}
This appendix is organized as follows:
\begin{itemize}[]
    \item In \cref{app:prel} we recall important notations and introduce some central inequalities. 
    \item In \cref{app:warmup} we link the length of the warm-up to the diameter of the set $\Theta$ it returns.
    \item In \cref{app:cs} we prove that $\mcal{C}_t(\delta)$ is a confidence region for $\theta_\star$. 
    \item In \cref{app:regret} we prove the different regret upper-bounds announced in the main paper
    \item In \cref{app:ccost} we detail the computational cost of the different approaches discussed in the main paper. 
    \item In \cref{app:aux} we list some auxiliary results, needed for the analysis.
    \item In \cref{app:exps} we provide additional numerical illustrations.
\end{itemize}

\vspace{1.5cm}
{\hypersetup{linkcolor=black}
\parttoc}
\vfill

\input{appendix/notations}

\input{appendix/warmup}

\input{appendix/confset}

\input{appendix/regret}

\input{appendix/cost}

\input{appendix/aux}
\input{appendix/exps}

\end{document}

%% file: content/intro.tex
\section{INTRODUCTION}
\paragraph{Logistic Bandit.}
The Logistic Bandit (\logb) framework describes sequential decision making problems in which an agent receives {structured binary} bandit feedback for her decisions. This namely allows to model numerous real-world situations where actions are evaluated by success/failure feedback (\emph{e.g.} click/no-click in ad-recommandation problems). From a theoretical standpoint, the \logb{} framework allows a neat and concise study of the interactions between non-linearity and the exploration/exploitation trade-off.  
Recent research efforts on this front were conducted by \cite{faury2020improved,abeille2020instance,jun2020improved}, relying on improved confidence sets for the design and analysis of regret-minimizing \logb{} algorithms. 
This has led to significant improvement over the seminal work of \cite{filippi2010parametric}, deflating the regret bounds by exponentially large factors. Their approach testifies of the importance of a careful handling of non-linearity in order to achieve optimal performances (\emph{i.e} matching the regret lower-bound from \citet[Theorem 2]{abeille2020instance}). 
From a learning-theoretic standpoint, this line of work brings the understanding of \logb{} almost to a tie with the  Linear Bandit (\lb). It highlights that some highly non-linear \logb{} instances are easier to solve (in some sense) than their \lb{} counterparts and brings forward algorithms with largely improved practical performances (see \citet[Section H]{abeille2020instance}). 
\paragraph{Limitations.}
A severe drawback of those improved \logb{} algorithms resides in their tremendous \emph{computational} cost. For instance, the \texttt{OFULog-r} algorithm of \cite{abeille2020instance} requires to maintain batch maximum-likelihood estimators (which cannot be updated recursively) and to solve at every round expensive convex programs.
The computational hardness of those tasks (respectively related to the \emph{learning} and \emph{planning} mechanisms of the algorithm) largely exceeds their \lb{} counterparts and lead to a painfully slow algorithm - prohibitively so for situations where decisions must be made on the fly. As a result, statistically efficient yet fast \logb{} algorithms are still missing - which is the topic of this paper.

\paragraph{Main Contributions.}
Our main contribution is \textbf{(1)} a new \emph{learning} procedure for \logb{}. It yields \textbf{(2)} a new confidence set which sufficient statistics can be maintained at each round with $\bigotilde{1}$ operations, without sacrificing statistical tightness. Furthermore \textbf{(3)} the shape of this set enables the deployment of efficient \emph{planning} strategies as a plug-in. This enables the design of computationally efficient algorithms whose regret guarantees match the lower-bound of \cite{abeille2020instance}. To the best of our knowledge, those \logb{} algorithms are the first to enjoy both statistical and computational efficiency simultaneously. We summarize our contributions in \cref{tab:regret_comparison}.

\paragraph{Organization.} We formally introduce the learning problem in \cref{sec:preliminaries} and discuss previous works, their limitations and remaining challenges. In \cref{sec:results} we describe our new estimation method, coined \underline{E}ffi\underline{c}ient L\underline{o}cal Learning for \underline{Log}istic Bandits (\texttt{ECOLog}). We then analyze an optimistic algorithm leveraging this procedure and claim that it enjoys the same regret guarantees obtained by \cite{abeille2020instance} while being critically less computationally hungry. We exhibit the main technical arguments needed to obtain this result and discuss potential extensions as well as limitations of our approach. In \cref{sec:dilution} we detail a variant of our algorithm, more complex but better suited for deployment in real-life situations. We provide similar guarantees for this algorithm and illustrate its good practical behavior with numerical simulations.

%% file: content/preliminaries.tex
\input{comp_table}

\section{PRELIMINARIES}
\label{sec:preliminaries}
\subsection{The Learning Problem}
\paragraph{Setting.} The \logb{} framework describes a repeated game between an agent and her environment. At each round, the agent selects an action (a vector in some Euclidean space) and receives a binary, Bernoulli distributed reward. More precisely, given an arm-set\footnote{For the sake of exposition we here only consider the static arm-set case. As later detailed, our results also apply to time-varying arm-sets and contextual settings.} $\mcal{A}\subset\mbb{R}^d$ the agent plays at each round $t$ an arm $a_t\in\mcal{A}$ and receives a stochastic reward $r_{t+1}$ following:
\begin{align}
	r_{t+1} \sim \text{Bernoulli}\left(\mu(a_t\transp\theta_\star)\right)\; ,\label{eq:model}
\end{align}
where $\mu(z) = (1+\exp(-z))^{-1}$ is the logistic function. The parameter $\theta_\star$ is unknown to the agent. We will work under the following assumption, standard for the study of \logb. 
\begin{ass}[Bounded Decision Set] 
    For any $a\in\mcal{A}$ we have $\|a\|\leq 1$. Also, $\|{\theta_\star}\|\leq S$ where $S$ is known. 
\end{ass}
Denote $a_\star \defeq \argmax_{a\in\mcal{A}} a\transp\theta_\star$ the best action in hindsight. The goal of the agent is to minimize her cumulative pseudo-regret up to time $T$: 
\begin{align*}
    \regret(T) \defeq T\mu\!\left(a_\star\transp\theta_\star\right)-\sum_{t=1}^T\mu\!\left(a_t\transp\theta_\star\right)\; .
\end{align*}

\paragraph{Reward Sensitivity.} Central to the analysis of \logb{} is the inverse minimal reward sensitivity $\kappa$. This \emph{problem-dependent} constant is defined as:
\begin{align*}
	\kappa \defeq 1/ \min_{a\in\mcal{A}}\min_{\|\theta\|\leq S}  \dot\mu(a\transp\theta) \; .
\end{align*}
Briefly, $\kappa$ measures the level of non-linearity of the reward signal, usually high in \logb{} problems. As such $\kappa$ is typically very large (numerically) even for reasonable configurations. We refer the reader to \citet[Section 2]{faury2020improved} for a detailed discussion on the importance of this quantity.

\paragraph{Additional Notations.} 
For any $t\geq 1$ we denote the $\mcal{F}_t\defeq \sigma(a_1,r_2, .., a_t)$ the $\sigma$-algebra encoding the information acquired after playing $a_t$ and before observing $r_{t+1}$. Throughout the paper time indexes reflect the measurability w.r.t $\mcal{F}_t$ (for example, $a_t$ is $\mcal{F}_t$-measurable but not $\mcal{F}_{t+1}$-measurable). For any pair $(x,y)\in\mbb{R}\!\times\!\{0,1\}$ we define:
\begin{align*}
	\ell(x,y) = -y\log\mu(x) - (1-y)\log(1-\mu(x))\; ,
\end{align*}
and the log-loss associated with the pair $(a_t, r_{t+1})$ writes $\ell_{t+1}(\theta) \defeq \ell(a_t\transp\theta, r_{t+1})$. Given a compact set $\Theta\subset\mbb{R}^d$ its diameter under the arm-set $\mcal{A}$ is:
\begin{align*}
	\diam_\mcal{A}(\Theta) = \max_{a\in\mcal{A}} \max_{\theta_1,\theta_2} \vert a\transp(\theta_1-\theta_2)\vert\; .
\end{align*}
We will use throughout the paper the symbol $\const$ to denote universal constants (\emph{i.e} independent of $S$, $\kappa$, $d$ or $T$) which exact value can vary at each occurrence. Similarly, we use the generic notation $\gamma_t(\delta)$ to denote various \emph{slowly} growing functions - more precisely such that $\gamma_t(\delta) = \const \text{poly}(S)d\log(t/\delta)$. The exact values for the different occurrences of such functions are carefully reported in the supplementary materials.

\subsection{Previous Work, Limitations and Remaining Challenges}

\bulletpoint{Filippi, then other work that are efficient but still sub-optimal} 
\subbulletpoint{give regret bound, insist that are over-explorative by design so usually do not perform well}
Being a member of the Generalized Linear Bandit family, the first algorithm for \logb{} was given by \cite{filippi2010parametric}. Their algorithm enjoys a regret scaling as $\bigotilde{\kappa d\sqrt{T}}$ - which although tight in $d$ and $T$, suffers from a prohibitive dependency in $\kappa$. Further, it is computationally inefficient as it requires the computation of a batch estimator for $\theta_\star$ at each round  (see \cref{app:cost_others} for a detailed discussion). This efficiency issue was fixed by \cite{zhang2016online,jun2017scalable,ding2021efficient} who proposed fully online estimation procedures. Their approaches however still suffer from detrimental dependencies in $\kappa$.

\paragraph{Statistical Optimality.} 
\bulletpoint{Faury, Abeille and Jun}
\subbulletpoint{improved initially by Faury. Then Abeille statistically optimal w.r.t $\kappa_\star$ and Jun improves for finite arm-sets}
The $\kappa$ dependency was trimmed by \cite{faury2020improved} who introduced an algorithm enjoying $\bigotilde{d\sqrt{T}}$ regret. Their approach defers the effect of non-linearity (embodied by $\kappa$) to a second-order term in the regret, dominated for large values of $T$. Similar results were also achieved by \cite{dong2019ts}, but only for the Bayesian regret. From a statistical viewpoint the story was closed by \cite{abeille2020instance} who proved a $\bigomega{d\sqrt{\dot\mu(a_\star\transp\theta_\star)T}}$ regret lower-bound and matching regret upper-bounds (up to logarithmic factors) for their algorithm \texttt{OFULog-r}. Given the typical scalings of $\kappa\propto\exp(\|\theta_\star\|)$ and $\dot\mu(a_\star\transp\theta_\star)\propto\exp(-\|\theta_\star\|)$ this deflates the regret of previous approaches by exponentially large factors.

\paragraph{Computational Cost.}
\subbulletpoint{OFULog-r: very inefficient because relies on the MLE estimator that requires batch computations. also planning is expensive for it requires solving one expensive convex problem for each arm. summarize in table}
Albeit statistically optimal, the algorithms  proposed by \cite{abeille2020instance} are strikingly computationally demanding and consequently prohibitively slow for practical situations. After inspection, two main computational bottlenecks of their approach emerge from their \emph{learning} and \emph{planning} mechanisms. From the learning side, they construct confidence regions of the form:
\begin{align}
	\left\{\theta, \; \| \theta - \hat\theta_{t}\|^2_{\mbold{H}_{t-1}(\theta)} \leq \gamma_t(\delta) \right\} \; ,\label{eq:prev_conf_set}
\end{align}
where $\hat\theta_{t}=\argmin_{\mbb{R}^d} \sum_{s=1}^{t-1} \ell_{s+1}(\theta)+\lambda\|\theta\|^2$ and 
\begin{align*}
\mbold{H}_t(\theta) = \sum_{s=1}^t \dot\mu(a_s\transp\theta)a_sa_s\transp+\lambda\mbold{I}_d\; .
\end{align*}
Those sufficient statistics are expensive to compute as both require a linear pass (at least) on the data. Note that simply testing whether a point lies in this set is costly - it requires $\bigomega{t}$ operations. The planning mechanism which leverages this confidence region suffers from this downside; to find an optimistic arm it must solve one expensive convex program per arm at every round. This program involves the complete log-loss, which evaluation also takes $\Omega(t)$ operations. Furthermore, bypassing optimism through randomized exploration (\emph{e.g.} Thompson Sampling) is particularly challenging as the results of \cite{agrawal2013thompson,abeille2017linear} do not apply to non-ellipsoidal confidence regions.

\bulletpoint{challenge}
\paragraph{Challenges.}
Our goal is to develop an \emph{efficient} algorithm (\emph{i.e} with reduced per-round computational cost)  which still enjoys statistical optimality (\emph{i.e} matches the lower-bound of \cite{abeille2020instance}). In light of the previous discussion, a crucial step is to derive an alternative to the confidence set from \cref{eq:prev_conf_set} which sufficient statistics can be updated at little cost. This must be done without sacrificing the confidence set's appreciation of the \emph{effective} reward sensitivity, captured by the matrix $\Hmat_t(\theta)$ and central for optimal performance. In other words, we seek to develop an efficient estimation procedure that captures the \emph{local} effects of non-linearity. This rules out merging the refined concentration tools of \cite{faury2020improved} with the online approaches of \cite{zhang2016online,jun2017scalable} which explicitly ressorts to \emph{global} quantities (\emph{e.g.} $\kappa$) in their estimation routines.

%% file: comp_table.tex
\begin{table*}[t]
    \centering
    \begin{tabular}{|c||c|c|c|c|}
    \hline
         \textbf{Algorithm} & \textbf{Regret Bound} & \textbf{Cost Per-Round} &  \textbf{Minimax} &  \textbf{Efficient}\\
         \hline
         \begin{tabular}{c} \texttt{GLM-UCB} \\ \cite{filippi2010parametric}\end{tabular} &  $\bigOtilde{\kappa d\sqrt{T}}$& $\bigO{d^2K+d^2T}$ & \crossmarkc & \crossmarkc \\ 
         \hline
         \begin{tabular}{c} \texttt{GLOC}, \texttt{OL2M} \\ \cite{jun2017scalable} \\ \cite{zhang2016online}\end{tabular} &  $\bigOtilde{\kappa d\sqrt{T}}$& $\bigO{d^2K}$ & \crossmarkc & \checkmarkc \\ 
         \hline
         \begin{tabular}{c} \texttt{OFULog-r} \\ \cite{abeille2020instance}\end{tabular} &  $\bigOtilde{d\sqrt{T\dot\mu(a_\star\transp\theta_\star)}}$& $\bigO{d^2KT}$ & \checkmarkc& \crossmarkc \\  
         \hline
          \begin{tabular}{c} \texttt{(ada-)OFU-ECOLog} \\ ({this paper})\end{tabular} &  $\bigOtilde{d\sqrt{T\dot\mu(a_\star\transp\theta_\star)}}$& $\bigOtilde{d^2K}$ & \checkmarkc & \checkmarkc \\
           \hline
    \end{tabular}
    \caption{Comparison of frequentist regret guarantees and computational cost for different \logb{} algorithms, on instances where $\vert\mcal{A}\vert=K<+\infty$. An algorithm is called minimax-optimal if it matches the regret lower-bound of \cite[Theorem 2]{abeille2020instance} and efficient if it matches the computational cost of \lb{} algorithms (up to logarithmic factors).}
    \label{tab:regret_comparison}
\end{table*}

%% file: content/results.tex
\section{MAIN RESULTS}\label{sec:results} 

\input{algos/pseudo_algo}

\input{algos/warm_up}

\input{algos/ecolog_pc}

In this section we present our approach to address the aforementioned challenges. We introduce \texttt{OFU-ECOLog}, an optimistic algorithm whose pseudo-code is provided in \cref{alg:ofuecolog}. It is built on top on three building blocks; \textbf{(1)} a short warm-up phase (forced-exploration) of size $\tau$ described in \cref{alg:warmup}, \textbf{(2)} the \texttt{ECOLog} estimation procedure described in \cref{alg:ecolog} and \textbf{(3)} an optimistic planning mechanism.

We provide in \cref{sec:tg} theoretical guarantees for the regret of \texttt{OFU-ECOLog} (\cref{thm:regret_ofuecolog}) and quantify its per-round computational cost (\cref{prop:per_round_cost}). It demonstrates that \texttt{OFU-ECOLog} enjoys both statistical and computational efficiency. 

Each building block (\textbf{1-3}) and their specific roles are detailed in subsequent sections. \Cref{sec:warmup_disc} is concerned with the initial forced-exploration  phase and its length $\tau$. \Cref{sec:ecolog_disc} details the estimation procedure \texttt{ECOLog} and the confidence region it induces. \Cref{sec:planning_disc} details the efficient deployment of the optimistic exploration strategy and describes the extension of \texttt{OFU-ECOLog} to \texttt{TS-ECOLog}, where optimism is replaced
with randomization.

\subsection{Statistical and Computational Efficiency}\label{sec:tg}
We claim the following result, which proof is deferred to \cref{app:regret_ofu}.

\begin{restatable}[Regret Bound]{thm}{thmregretoful}\label{thm:regret_ofuecolog}
\ifdetail
Let $\delta\in(0,1]$. Setting $\tau= \const \kappa S^6 d^2 \log(T/\delta)^2$ ensures the regret of \textnormal{\texttt{OFU-ECOLog}($\delta,\tau$)} satisfies with probability at least $1-2\delta$:
\begin{align*}
	\regret(T) \leq \const Sd\log(T/\delta)\sqrt{T\dot\mu(a_\star\transp\theta_\star)} + \const S^6\kappa d^2 \log(T/\delta)^2\; .
\end{align*}
\else
Let $\delta\in(0,1]$. Setting $\tau\!=\!\kappa S^6\gamma_T(\delta)^2$ ensures the regret of \textnormal{\texttt{OFU-ECOLog}($\delta,\tau$)} satisfies with probability at least $1-2\delta$:
\begin{align*}
	\regret(T) \leq \const Sd\sqrt{T\dot\mu(a_\star\transp\theta_\star)}\log(T/\delta) \\ + \const S^6\kappa d^2\log(T/\delta)^2\; .
\end{align*}
\fi
\end{restatable}

As promised the dominating term in \texttt{OFU-ECOLog}'s regret-bound matches the lower-bound of \cite{abeille2020instance} and scales with the reward sensitivity at the best action $a_\star$. Further, the second-order term identically matches its counterpart from previous work in its scaling w.r.t $d$, $T$ and $\kappa$. This establishes the statistical efficiency and we now move up to the computational cost. 
For this, we claim the following bound on the complexity of \texttt{OFU-ECOLog}.

\begin{restatable}[Computational Cost]{prop}{propperroundcost}\label{prop:per_round_cost}
	Let $\vert \mcal{A}\vert = K <\infty$. Each round $t$ of \textnormal{\texttt{OFU-ECOLog}} can be completed within $\bigo{Kd^2+ d^2\log(t)^2}$ operations.
\end{restatable}

The proof is deferred to \cref{app:cost}. This result mainly relies on the fact that the \textnormal{\texttt{ECOLog}} routine (\cref{alg:ecolog}) solves convex programs that are cheap (\emph{i.e} for which gradients are inexpensive to compute) and that can be efficiently preconditioned. Furthermore \textnormal{\texttt{OFU-ECOLog}} leverages ellipsoidal confidence sets, for which optimism can be efficiently enforced (at least for finite arm-sets). This fulfills our promise of computational efficiency.

\subsection{Warm-Up}\label{sec:warmup_disc}
One of the main challenge to avoid prohibitive exponential dependencies in \logb{} is to tightly control the reward sensitivity across $\mcal{A}\times\Theta$ - that is, without resorting to global problem-dependent constants (\emph{e.g.} $\kappa$). Following \cite{faury2020improved} a first useful step in that direction is to leverage the self-concordance property of the logistic function. It ensures that for any $a\in\mcal{A}$:
\begin{align}
	 \forall\theta_1,\theta_2\in\Theta,\;  \dot\mu(a\transp\theta_1) \leq \dot\mu(a\transp\theta_2)\exp(\diam_\mcal{A}(\Theta))\; .\label{eq:sc_link}
\end{align}
The role of the warm-up phase is to identify a set $\Theta$ containing $\theta_\star$ (with high probability) and which diameter is a constant, independent of problem-dependent quantities (\emph{e.g.} $\|\theta_\star\|$ or $S$). The warm-up mechanism described in \cref{alg:warmup} constructs such a set $\Theta$ which diameter is controlled through the length $\tau$ of this forced-exploration phase. In particular, we show that if $\tau\!\propto\!\kappa$ (\cref{prop:diameter_warmup} in the appendix):
\begin{align}
	\diam_\mcal{A}(\Theta)\leq 1 \; \text{for} \; \Theta\leftarrow \text{\warmup{\tau}}\; .\label{eq:diam_claim}
\end{align}
Combining \cref{eq:sc_link,eq:diam_claim} allows to control the reward sensitivity across the set $\Theta$ at little cost (\emph{i.e} independent of problem-dependent constants). Theoretically speaking, the regret incurred during the warm-up phase forms a second-order term, dominated in the overall regret bound. From a practical perspective however, resorting to forced-exploration is inconvenient - a downside we address in \cref{sec:dilution}. 

\subbulletpoint{refined warm up for contextual case, give reference}
\begin{rem*}[Optimal Design]
There exists alternatives warm-up strategies which ensures similar guarantees - see for instance \cite{jun2020improved} for a solution based on optimal design. It involves more complex mechanisms to reduce the length $\tau$ - we stick here to a simple strategy for the sake of exposition. 
\end{rem*}

\subsection{Efficient Local Learning} \label{sec:ecolog_disc}
We now describe \texttt{ECOLog}, a new estimation routine summarized in \cref{alg:ecolog} which is at the core of the online construction of tight confidence sets. It operates on the convex set $\Theta$ returned by the warm-up procedure. It maintains estimates $\{\theta_{t}\}_t$ of $\theta_\star$ following the update rule:
\begin{align}
	&\theta_{t+1} = \argmin_{\theta\in\Theta} \Big[\eta\left \| \theta - \theta_{t}\right\|_{\Wmat_t}^2 + \ell_{t+1}(\theta)\Big] \; , \label{eq:def_theta_again} \\
	\text{where } &\Wmat_{t} = \sum_{s=1}^{t-1}\dot\mu(a_s\transp\theta_{s+1})a_sa_s\transp + \lambda\mbold{I}_d \; .\notag 
\end{align}
The learning rate $\eta$ is tied to the diameter $\diam_\mcal{A}(\Theta)$ of the decision set $\mcal{A}\times\Theta$. After round $t$, the next estimate $\theta_{t+1}$ minimizes an approximation of the true cumulative log-loss $\sum_{s=1}^{t} \ell_{s+1}(\theta)$ that is decomposed in two terms. The first consists in a quadratic proxy for the past losses constructed through the sequence $\{\theta_s\}_{s\leq t}$. It is designed to incorporate the information acquired so far in the update since:
\begin{align*}
	\argmin_\theta  \sum_{s=1}^{t-1} \ell_{s+1}(\theta) \approx \argmin_\theta \left\| \theta - \theta_t \right\|^2_{\Wmat_{t}} \; .
\end{align*}

On the other hand, the second term is the instantaneous log-loss $\ell_{t+1}(\theta)$ which accounts for the novel information of the pair $(a_t,r_{t+1})$. The motivation behind the overall structure of the update is the following: while the cumulative log-loss is strongly convex and can therefore be well approximated by a quadratic function, the instantaneous loss $\ell_{t+1}$ has flat tails which cannot be captured by a quadratic shape. 

\begin{rem*}[Comparison with ONS]
While at first glance it resembles the Online Newton Step (ONS) mechanisms used by \cite{zhang2016online,jun2017scalable} there are two important differences. First, the update is driven by the matrix $\Wmat_{t}$ which relies on the \emph{estimated} reward sensitivity, and not on its worst-case alternative $\kappa$. Second, we do not rely on (potentially loose) approximations for $\ell_{t+1}$. This rules out having access to a closed-form for $\theta_{t+1}$. 
\end{rem*}

Since the solution of \cref{eq:def_theta_again} does not admit a closed-form expression, one can only solve it up to an $\eps$ accuracy (\emph{e.g.} with projected gradient descent). Formally, we compute estimators $\theta'_{t+1}$ such that $\|\theta'_{t+1} - \theta_{t+1}\| \leq \eps$. The following statement guarantees that this can be done at little cost. 
\begin{prop}[Computational Cost]\label{prop:mt_cost_ecolog}
	Running \textnormal{\texttt{ECOLog}} up to $\eps>0$ accuracy requires $\bigo{d^2\log(1/\eps)^2}$ operations.
\end{prop}
For the sake of exposition, we ignore optimization errors in the following since $\eps$ can be arbitrarily small. The induced errors and their propagation are addressed in formal proofs in the supplementary. 

Finally, the use of \texttt{ECOLog} at each round within \cref{alg:ofuecolog} yields a sequence $\{\theta_t, \Wmat_t\}_t$ associated with the sets:
\begin{align*}
	\mcal{C}_t(\delta) \defeq \left\{ \theta, \left\|\theta - \theta_t\right\|^2_{\Wmat_t}\leq \gamma_t(\delta) \right\}\;,
\end{align*}
which are confidence regions for $\theta_\star$. 

\begin{restatable}[Confidence Set]{prop}{thmourconfset}\label{prop:confset}
\ifdetail
Let $\delta\in(0,1]$ and $\{(\theta_t, \Wmat_t)\}_t$ the parameters maintained by \cref{alg:ofuecolog} with $\tau$ set according to \cref{prop:diameter_warmup}. Then:
\begin{align*}
	\mbb{P}\left(\forall t\geq 1, \; \big\|\theta_\star-\theta'_{t+1}\big\|^2_{\Wmat_{t+1}}\leq \sigma_t(\delta)\text{ and }\theta_\star\in \Theta\right)\geq 1-2\delta\; .
\end{align*}
\else
Under the conditions of \cref{thm:regret_ofuecolog}:
\begin{align*}
	\mbb{P}\!\left( \forall t \geq \tau, \; \theta_\star \in \mcal{C}_t(\delta)  \right) \geq 1\!-\delta\; .
\end{align*}
\fi
\end{restatable}

\Cref{prop:confset} emulates the original concentration results of \cite{faury2020improved} (see \cref{eq:prev_conf_set}). The matrix $\Wmat_t$ stands as an on-policy proxy for the ``correct'' concentration metric $\Hmat_t(\theta_\star)$. This ultimately preserves statistical tightness (\cref{thm:regret_ofuecolog}) but with sufficient statistics that are now updated online.

\paragraph{Proof Sketch.} We provide here the key technical arguments behind the derivation of \cref{prop:confset}. It is inspired and shares close connections with the work of ~\cite{jezequel2020efficient} - which was conducted for an Online Convex Optimization setting.\\
A crucial ingredient for our analysis is a \emph{local} quadratic lower-bound\footnote{Similar bounds appear in \cite{jezequel2020efficient,abeille2020instance} but are used for different purposes.} for the logistic loss, stating that for any $\theta\in\Theta$:
\begin{align*}
	\ell_{t+1}(\theta_\star) \gtrsim \ell_{t+1}(\theta) + \nabla\ell_{t+1}(\theta)\transp(\theta_\star-\theta)  \\ + \dot\mu(a_t\transp\theta)(a_t\transp(\theta_\star-\theta))^2\; .
\end{align*}
Notice how the above does not depend on any \emph{global} quantities (\emph{e.g.} $S$ or $\|\theta_\star\|$). It allows to tie the parameters uncertainty $\| \theta_\star -  \theta_{t+1}\|^2_{\Wmat_{t+1}}$ to the excess cumulative loss in $\{\theta_{s+1}\}_s$. Indeed algebraic manipulations lead to:
\begin{align*}
	\left\| \theta_\star -  \theta_{t+1} \right\|^2_{\Wmat_{t+1}} \lesssim \sum_{s=1}^{t} \ell_{s+1}(\theta_\star) - \ell_{s+1}(\theta_{s+1})\; . 
\end{align*}
We are therefore left to bound the r.h.s. To do so we introduce an intermediary parameter: 
\begin{align*}
	\bar\theta_s = \argmin_\Theta \eta \|\theta-\theta_s\|^2_{\Wmat_s} + \ell(a_s\transp\theta, 0) + \ell(a_s\transp\theta, 1)\; ,
\end{align*}
and decompose the sum to control as follows:
\begin{align}\label{eq:loss_decomp_main}
	 \sum_{s=1}^t \ell_{s+1}(\theta_\star) \!-\! \ell_{s+1}(\bar\theta_{s})+ \sum_{s=1}^t\ell_{s+1}(\bar\theta_{s})  \!-\! \ell_{s+1}(\theta_{s+1}) \; .
\end{align}
The parameter $\bar\theta_s$ is a $\mcal{F}_s$-measurable version of $\theta_{s+1}$, regularized in the last direction $a_s$ by two logistic losses fitting antipodal rewards ($r_{s+1}=0$ and $1$). The first term in \cref{eq:loss_decomp_main} is tied to the stochastic nature of the observations and is bounded using the concentration inequality of \citet[Theorem 1]{faury2020improved}. With probability at least $1-\delta$,
\begin{align*}
	\sum_{s=1}^t \ell_{s+1}(\theta_\star) - \ell_{s+1}(\bar\theta_{s}) \lesssim  \log(t/\delta)\;.
\end{align*}
Bounding the second term requires quantifying the deviation between $\theta_{s+1}$ and its $\mcal{F}_{s}$-measurable counterpart $\bar\theta_s$. Leveraging convexity leads to the sequence of inequalities:

\begin{align*}
	\sum_{s=1}^t\ell_{s+1}(\bar\theta_{s})  \!-\! \ell_{s+1}(\theta_{s+1}) &\leq \sum_{s=1}^t \dot\mu(a_s\transp\bar\theta_s) \| a_s\|^2_{\Wmat^{-1}_{t+1}}\\
	&\lesssim \sum_{s=1}^t \dot\mu(a_s\transp\theta_{s+1}) \| a_s\|^2_{\Wmat^{-1}_{t+1}}\\
	&\lesssim d\log(t)\; .
\end{align*} 

The second inequality is obtained by relating the reward sensitivities $\dot\mu(a_s\transp\bar\theta_s)$ and $\dot\mu(a_s\transp\theta_{s+1})$. Both are comparable thanks to the warm-up procedure. Indeed from \cref{eq:sc_link,eq:diam_claim},
\begin{align}\label{eq:ps_reward_sensitivity}
	\dot\mu(a_s\transp\bar\theta_s) \leq \exp(\diam_\mcal{A}(\Theta))  \dot\mu(a_s\transp\theta_{s+1}) \lesssim \dot\mu(a_s\transp\theta_{s+1})\; .
\end{align}
The last inequality directly follows from the Elliptical Potential Lemma (see \cref{lemma:ellipticalpotential}).

\begin{rem*}[Warm-Up and Online Newton Step]
It is natural to wonder whether the ONS-like approaches of \cite{zhang2016online,jun2017scalable} could also benefit from the refined parameter set returned by the warm-up procedure. As detailed in \cref{app:ons_warmup} this is not the case. Their respective methods hard-code \emph{global} quantities within their updates steps (such as the minimum curvature of the log-loss, or the exp-concavity constant). Those are related to $\kappa$ and cannot be removed even when operating close to $\theta_\star$. 
\end{rem*}

\subsection{Exploration Strategy}\label{sec:planning_disc}
\paragraph{Optimistic Exploration.} \texttt{OFU-ECOLog} builds on $\mcal{C}_t(\delta)$ (the confidence set of \cref{prop:confset}) to find an optimistic arm. Formally, it prescribes playing:
\begin{align*}
	a_t \in \argmax_{a\in\mcal{A}} \max_{\theta\in\mcal{C}_t(\delta)} a\transp\theta\; .
\end{align*}
A solution for this program might be expensive to compute in general. However, the ellipsoidal nature of $\mcal{C}_t(\delta)$ may simplify this task as it allows for an equivalent definition of $a_t$ as:
\begin{align*}
	a_{t} \in \argmax_{a\in\mcal{A}} a\transp\theta_{t} + \sqrt{\gamma_t(\delta)} \|a\|_{\Wmat_{t}^{-1}} \; .
\end{align*}
For finite arm-sets ($\vert \mcal{A}\vert = K \!<\!+\infty$) this program can be solved by enumerating over the arms - bringing the total cost of the optimistic planning to $\bigo{d^2 K}$.

\paragraph{Thompson-Sampling extension.} 
The shape of $\mcal{C}_{t+1}(\delta)$ also enables the use of \emph{randomized} exploration mechanisms in a principled fashion. For instance, Thompson Sampling (TS) replaces the burden to find an optimistic parameter by sampling in slightly inflated confidence sets (see \cite{abeille2017linear}). It is often preferred in practical applications for its simplicity and good empirical performances. It also allows to deal with infinite arm-sets, whenever an oracle for computing $a_\star(\theta) = \argmax_{a\in\mcal{A}} a\transp\theta$ is cheaply available for any $\theta$ (\emph{e.g.} when the action space is the unit-ball $\mcal{B}_d$). We introduce \texttt{TS-ECOLog}  in \cref{app:regret_ts}, a TS version of \texttt{OFU-ECOLog}. It enjoys similar regret bounds, but inflated by a $\sqrt{d}$ factor (as in the \lb{} case). The algorithm displays little conceptual novelty compared to its linear counterpart, but its analysis requires additional technical care to prove its statistical efficiency.  Overall, this answers positively the question opened by \cite{faury2020improved} about the extension of their approach to randomized strategies.

%% file: algos/pseudo_algo.tex
\begin{algorithm*}[tb]
   \caption{\texttt{OFU-ECOLog}}
   \label{alg:ofuecolog}
	\begin{algorithmic}
	\REQUIRE{failure level $\delta$, warm-up length $\tau$.}
	\STATE Set $\Theta \leftarrow \texttt{WarmUp}(\tau)$ (see \cref{alg:warmup}). \hfill\COMMENT{forced-exploration}
   \STATE Initialize $\theta_{\tau+1}\in\Theta$, $\Wmat_{\tau+1} \leftarrow \mbold{I}_d$ and $\mcal{C}_{\tau+1}(\delta)\leftarrow \Theta$.
   \FOR{$t\geq \tau+1$}
    \begin{spacing}{1.2}
   \STATE Play $a_t \in \argmax_{a\in\mcal{A}} \max_{\theta\in\mcal{C}_t(\delta)} a\transp\theta\; .$ \hfill\COMMENT{planning} 
   \STATE Observe reward $r_{t+1}$, construct loss $\ell_{t+1}(\theta)=\ell(a_t\transp\theta, r_{t+1})$.
   \STATE  Compute $(\theta_{t+1},\, \Wmat_{t+1})\leftarrow \texttt{ECOLog}(1/t, \Theta,\ell_{t+1},\Wmat_t,\theta_t)$ (see \cref{alg:ecolog}).\hfill\COMMENT{learning}
   \end{spacing}
   \vspace{2pt}
   \STATE Compute $\mcal{C}_{t+1}(\delta) \leftarrow\left \{\left\| \theta-\theta_{t+1} \right\|^2_{\Wmat_{t+1}} \leq \gamma_t(\delta)\right\}$.
   \ENDFOR
\end{algorithmic}
\end{algorithm*}

%% file: algos/warm_up.tex
\begin{procedure}[tb]
   \caption{\texttt{WarmUp}}
   \label{alg:warmup}
	\begin{algorithmic}
	\REQUIRE{length $\tau$.}
	\STATE Set $\lambda\leftarrow \gamma_\tau(\delta)$, initialize $\Vmat_0\leftarrow\lambda \mbold{I}_d$.
   \FOR{$t\in[1,\tau]$}
   	\STATE Play $a_t\in\argmax_{\mcal{A}} \| a\|_{\Vmat^{-1}_{t-1}}$, observe $r_{t+1}$. 
	\STATE Update $ \Vmat_t \leftarrow \Vmat_{t-1} + a_ta_t\transp/\kappa$.
   \ENDFOR
   \STATE Compute $\hat\theta_{\tau+1} \leftarrow \argmin_\theta \sum_{s=1}^\tau \ell_{s+1}(\theta) + \lambda\|\theta\|^2$.
   \ENSURE $\Theta = \left\{ \theta, \left\|\theta-\hat\theta_{\tau+1}\right\|^2_{\Vmat_{\tau}} \leq{\gamma_{\tau}(\delta)}\right\}$.
   \end{algorithmic}
\end{procedure}

%% file: algos/ecolog_pc.tex
\begin{procedure}[tb]
   \caption{\texttt{ECOLog} }
   \label{alg:ecolog}
	\begin{algorithmic}
	\REQUIRE{accuracy $\varepsilon$, convex set $\Theta$, $\ell_{t+1}$, $\Wmat_t$, $\theta_t$.}
		\STATE Compute $D\leftarrow \diam_{\mcal{A}}(\Theta)$, set $\eta\leftarrow (2+D)^{-1}$.
		\STATE Solve to precision $\varepsilon$:
		$$\theta_{t+1} = \argmin_{\theta\in\Theta} \Big[\eta \left\| \theta-\theta_t\right\|^2_{\Wmat_t}+ \ell_{t+1}(\theta)\Big]\; .$$
		\vspace{-10pt}
		\STATE Update $\Wmat_{t+1}\leftarrow \Wmat_t + \dot\mu(a_t\transp\theta_{t+1})a_ta_t\transp$. 
		\ENSURE $\theta_{t+1}, \Wmat_{t+1}$
   \end{algorithmic}
\end{procedure}

%% file: content/dilute.tex
\input{algos/diluted_pseudo_algo}

\section{REMOVING THE WARM-UP} \label{sec:dilution}

\paragraph{Practical Limitations.} Despite being rather common in the Generalized Linear Bandit literature (\emph{e.g.} \citep{li2017provably,kveton2020randomized,jun2020improved,ding2021efficient}) the use of warm-up phases is concerning from a practical stand-point. Indeed \textbf{(1}) it \emph{hard-codes} a forced-exploration regime lasting at least $\kappa$ rounds at the beginning of any experiment. Given the typical scaling of $\kappa$ in practical situations this implies that the algorithm selects actions at random for the first few thousand steps. While it only impacts low-order terms in the regret bound, it is problematic to suffer this price by design, even when not necessary (see \citet[Section 4]{abeille2020instance}). Furthermore,  \textbf{(2)} generalizing warm-up phases to handle \emph{contextual} arm-sets requires adopting strong distributional assumptions on the contexts - leaving out the case where an adversary picks context.

\paragraph{Data-Driven Alternative.} We relied so far on warm-up phases to isolate the different challenges (locality, efficiency, statistical tightness). We now switch gears and propose a refined approach which addresses the issues raised by forced-exploration - however at the cost of a more intricate algorithm. At the heart of this refinement lies a \emph{data-dependent} version of our confidence set. This allows  \textbf{(1)} the design an \emph{adaptive} mechanism which preserves statistical efficiency while ultimately removing the need for forced-exploration. Furthermore, it \textbf{(2)} extends the regret bound derived in \cref{sec:results} to the contextual case, without requiring any distributional assumptions on the exogenous contexts. To our knowledge, this is a first for approaches resorting to warm-ups.

\subsection{An Adaptive Approach}

\paragraph{Intuition.} As highlighted by the proof sketch from \cref{sec:ecolog_disc}, the warm-up phase allows to tightly control the radius of $\mcal{C}_t(\delta)$. More precisely, it constructs a {small} admissible set $\Theta$ that constrains the reward sensitivities $\dot\mu(a_s\transp\bar\theta_s)$ and $\dot\mu(a_s\transp\theta_{s+1})$ to be comparable for all $s$ (see \cref{eq:ps_reward_sensitivity}). 

A naive way to remove the need for enforcing this property \emph{a priori} would be to reject \emph{on-the-fly} points that don't conform with the following condition: 
\begin{align}\tag{C$_{0}$}
	\dot\mu(a_s\transp\bar\theta_s) \leq 2\dot\mu(a_s\transp\theta_{s+1}) \; , \label{eq:cond}
\end{align}

This high-level idea is behind the design of our adaptive mechanism, detailed below. 

\paragraph{Adaptive Mechanism.} Given $(a_s,r_{s+1})$ if the associated $\theta_{s+1}$ breaks \eqref{eq:cond} we do not use it to update our current estimate. Instead we leverage this information to ensure that \eqref{eq:cond} is more likely to hold in the future. We maintain $\mcal{H}_{s+1} = \{a_l,r_{l+1}\}_{l\leq s}$ formed by pairs rejected up to round $s$ and compute:
\begin{align*}
	\hat\theta^{\mcal{H}}_{s+1} \in &\argmin\sum_{(a,r)\in\mcal{H}_{s+1}} \ell(a\transp\theta,r)+\gamma_s(\delta)\|\theta\|^2\; ,
\end{align*}
and $\Vmat^\mcal{H}_s = \sum_{a\in\mcal{H}_{s}} aa\transp/\kappa+\gamma_s(\delta)\mbold{I}_d$. We use this to build the parameter set:
\begin{align*}
	\Theta_{s+1} =\left\{ \theta,\; \big\| \theta -\hat\theta^{\mcal{H}}_{s+1} \big\|_{\Vmat^\mcal{H}_s}^2 \leq \gamma_s(\delta)\right\} \; .
\end{align*} 
We will use this convex set in the \texttt{ECOLog} procedure for subsequent rounds. As points are being added to $\{\mcal{H}_s\}_s$ the sequence of $\{\Theta_s\}_s$ deflates. The downstream estimates $\{\theta_{s+1},\bar\theta_s\}_{s}$ are therefore closer and \eqref{eq:cond} is more likely to hold. The key to assert the validity of this mechanism is to ensure that this sequential refinement does not occur too often. 

\begin{figure*}[!th]
\begin{subfigure}{0.33\textwidth}
	\includegraphics[width=\textwidth]{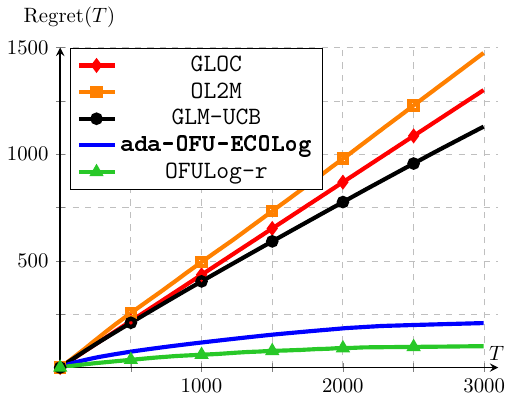}
	\caption*{$d=2$, $\vert\mcal{A}\vert=20$, $\kappa=400$}
\end{subfigure}
\begin{subfigure}{0.33\textwidth}
	\includegraphics[width=\textwidth]{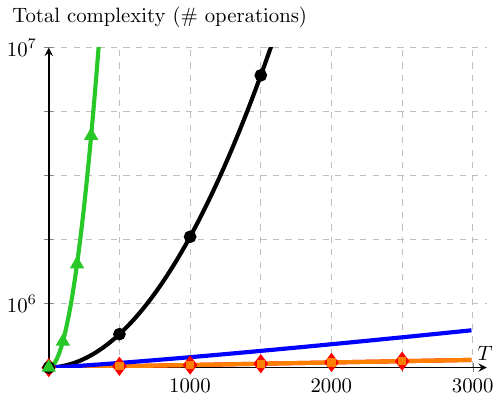}
	\caption*{$d=2$, $\vert\mcal{A}\vert=20$, $\kappa=400$}
\end{subfigure}
\begin{subfigure}{0.33\textwidth}
	\includegraphics[width=\textwidth]{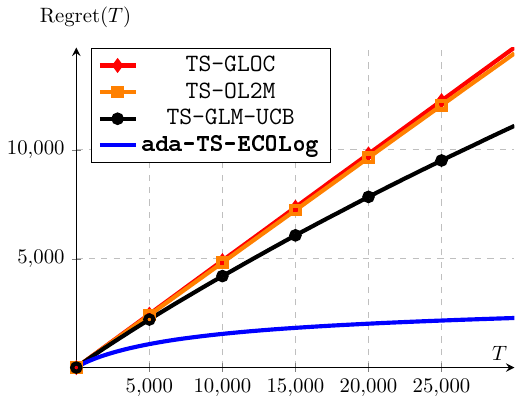}
	\caption*{$d=5$, $\mcal{A}=\mcal{B}_5$, $\kappa=400$.}
\end{subfigure}
\caption{Numerical simulations on \logb{} problems. We implement algorithms as prescribed by theory (\emph{e.g.} we do not tune exploration) and average regret curves over 100 independent trajectories. \emph{(left)} Regret curves on a two-dimensional \logb{} problem with 20 arms, sampled at random within the unit ball. We chose a small number of arms along with a short horizon to allow \texttt{OFULog-r} to run in a reasonable time. \emph{(center)} Overall complexity of the different algorithms for this same instance. As hinted by the regret and complexity bounds, \texttt{ada-OFU-ECOLog} is the only one displaying good performances and at little computational cost. \emph{(right)} Numerical simulations with infinite arm-set (5-dimensional unit-ball) for which we evaluate the TS version of each algorithm (this excludes \texttt{OFULog-r} which does not have a straightforward TS extension).}\label{fig:comp}
\end{figure*}

\paragraph{Technical Adjustment.} The idea presented above needs a slight technical refinement to bear a principled algorithm. \eqref{eq:cond} prescribes filtering the arm $a_s$ according to $\theta_{s+1}$, an $\mcal{F}_{s+1}$-adapted quantities. This breaches concentration properties we need to prove low-regret. To circumvent this issue we fall back on an $\mcal{F}_s$-adapted condition covering the potential values of $\theta_{s+1}$ (depending on the realization of $r_{s+1}$). Let:
\begin{align}\label{eq:defbaru}
	\theta^u_{s} = \argmin_{\theta\in\Theta_s}\left[ \eta\left \| \theta - \theta_{s}\right\|_{\Wmat_s}^2 + \ell(a_s\transp\theta, u)\right]\; ,
\end{align}
for $u\in\{0,1\}$.
Note that $\theta_{s+1}$ is either $\theta^0_{s}$ or $\theta^1_{s}$. We replace \eqref{eq:cond} by the condition:
\begin{align}\label{eq:cond_measurable}\tag{C$_1$}
	\dot\mu(a_s\transp\bar\theta_s) \leq 2\dot\mu(a_s\transp\theta_{s}^u), \; \;  \forall u\in\{0,1\}\; .
\end{align}
The algorithm \texttt{ada-OFU-ECOLog} presented in \cref{alg:dilutedofuecolog} combines this adjustment with the aforementioned adaptive mechanism. 
\subsection{Theoretical Guarantees}

\paragraph{Regret Bound.} Thanks to its adaptivity, we can claim regret guarantees for \texttt{ada-OFU-ECOLog} in contextual settings - \emph{i.e} that holds for \emph{any} sequence of time-varying arm-set $\{\mcal{A}_t\}_{t\geq 1}$. 

\begin{restatable}{thm}{regretdiluted}\label{thm:dilutedofuecolog}
Let $\delta\!\in\![0,1)$. With probability at least $1\!-\!\delta$ the regret of \textnormal{\texttt{ada-OFU-ECOLog}($\delta$)} satisfies:
\ifdetail
\begin{align*}
\regret(T) \leq \const Sd\sqrt{\sum_{t=1}^T\dot\mu(a_{\star,t}\transp\theta_\star)}\log(T/\delta)  + \const S^6\kappa d^2\log(T/\delta)^2\; ,
\end{align*}
\else
\begin{align*}
	\regret(T) \leq \const Sd\sqrt{\sum_{t=1}^T\dot\mu(a_{\star,t}\transp\theta_\star)}\log(T/\delta) \\ + \const S^6\kappa d^2\log(T/\delta)^2\; ,
\end{align*}
\fi
where $a_{\star,t} = \argmax_{a\in\mcal{A}_t} a\transp\theta_\star$. 
\end{restatable}
The proof is deferred to \cref{app:ada_ofu_regret}. This result establishes similar (although more general) regret guarantees than \cref{thm:regret_ofuecolog}. The two bounds are identical for constant arm-sets ($\mcal{A}_t\equiv \mcal{A}$). In the contextual case, the leading term is $\sqrt{T}\sqrt{\sum_{t=1}^T\dot\mu(a_{\star,t}\transp\theta_\star)/T}$, replacing the reward sensitivity at the optimal action by its on-trajectory average version. 

\paragraph{Computational Cost.} The per-round computational cost of \cref{alg:dilutedofuecolog} is larger than \texttt{OFU-ECOLog} as it sometimes requires $\bigo{\vert \mcal{H}_t\vert}$ extra operations to compute $\hat\theta^\mcal{H}_t$. The computational overhead is small as we can prove that $\vert \mcal{H}_t\vert \lesssim \kappa$. 

\begin{restatable}{prop}{propdilutedcomplexity}\label{prop:dilutedcomplexity}
The per-round computational cost of \cref{alg:dilutedofuecolog} is bounded by $\bigO{\kappa + Kd^2+ d^2\log(T)^2}$.
\end{restatable}

This extra computation is needed only when violating \eqref{eq:cond_measurable} which happens at most for $\kappa$ rounds. 

\paragraph{Adaptivity in Practice.}\cref{thm:dilutedofuecolog,prop:dilutedcomplexity} establish that hard-coding a warm-up phase can be avoided with little to no impact on the worst-case performance or computational cost.  The adaptive nature of \texttt{ada-OFU-ECOLog} allows to enjoy stronger empirical performances in ``nice'' configurations. For instance, in all the numerical experiments that follow we found that the condition \eqref{eq:cond_measurable}  was \emph{never} triggered. In such cases, \cref{alg:dilutedofuecolog} simply reduces to \texttt{OFU-ECOLog} \emph{without} any forced-exploration. This is consistent with the analysis of \citet[Section 4]{abeille2020instance} which suggests that low-order $\kappa$ dependencies (introduced here by the warm-up) can sometimes be avoided.

\subsection{Numerical Simulations}

The numerical illustrations presented in \cref{fig:comp} are consistent with our theoretical findings summarized in \cref{tab:regret_comparison}.\footnote{For reproducing experiments, see \url{https://github.com/criteo-research/logistic_bandit}.} As predicted, \texttt{ada-OFU-ECOLog} enjoys best-of-both-worlds properties by displaying small regret and small computational cost. Additional numerical illustrations for \logb{} instances of higher dimensions can be found in \cref{app:exps}.\\
The value of $\kappa$ for \logb{} instances we consider are reasonable (comparable to real-life situations). Still, it precludes the use of warm-up phase in practice for it will simply last longer than the horizon we consider (for which \texttt{OFULog-r} and \texttt{ada-OFU-ECOLog} already exhibits asymptotic behavior). Note that this is even worse for other approaches using forced-exploration \citep{kveton2020randomized,ding2021efficient} as their respective warm-ups are typically even longer ($\propto \kappa ^2$). Finally, we report results for an infinite arm-set for which our approach yields the only tractable algorithm enjoying statistical efficiency.

%% file: algos/diluted_pseudo_algo.tex
\begin{algorithm*}[tb]
   \caption{\texttt{ada-OFU-ECOLog}}
   \label{alg:dilutedofuecolog}
	\begin{algorithmic}
	\REQUIRE{failure level $\delta$.}
   \STATE Initialize $\Theta_1 = \{ \| \theta \| \leq S \}$, $\mcal{C}_1(\delta) \leftarrow \Theta_1$, $\theta_{1}\in\Theta$, $\Wmat_{1} \leftarrow \mbold{I}_d$ and $\mcal{H}_1\leftarrow \emptyset$. 
   \FOR{$t\geq 1$}
    \begin{spacing}{1.2}
    \STATE Play $a_t \in \argmax_{a\in\mcal{A}} \max_{\theta\in\mcal{C}_t(\delta)} a\transp\theta$, observe reward $r_{t+1}$.
    \STATE Compute the estimators $\theta_t^0$, $\theta_t^1$ (see \cref{eq:defbaru}) and $\bar\theta_t$. 
    \IF{$\dot{\mu}(a_t\transp\bar\theta_t)\leq 2\dot{\mu}(a_t\transp\theta_t^0)$ and $\dot{\mu}(a_t\transp\bar\theta_t)\leq 2\dot{\mu}(a_t\transp\theta_t^1)$}
    	\STATE Form the loss $\ell_{t+1}$ and compute $(\theta_{t+1},\, \Wmat_{t+1})\leftarrow\texttt{ECOLog}(1/t, \Theta_t,\ell_{t+1},\Wmat_t,\theta_t)$.
	\STATE Compute $\mcal{C}_{t+1}(\delta) \leftarrow \left\{\| \theta-\theta_{t+1} \|^2_{\Wmat_{t+1}} \leq \gamma_t(\delta)\right\}$, set $\mcal{H}_{t+1} \leftarrow \mcal{H}_t$.
	\ELSE
		\STATE Set $\mcal{H}_{t+1} \leftarrow \mcal{H}_t\cup \{ a_t, r_{t+1} \}$ and compute $\hat\theta^{\mcal{H}}_{t+1} = \argmin \sum_{(a,r)\in\mcal{H}_{t+1}} \ell(a\transp\theta, r) + \gamma_t(\delta)\|\theta\|^2$.
		\STATE  Update $\mbold{V}^{\mcal{H}}_{t}\leftarrow \sum_{a\in\mcal{H}_{t+1}} aa\transp/\kappa + \gamma_t(\delta)\mbold{I}_d$, $\theta_{t+1}\leftarrow \theta_t$ and $\Wmat_{t+1}\leftarrow \Wmat_t$.
		\STATE Compute $\Theta_{t+1} = \left\{ \| \theta - \hat\theta^{\mcal{H}}_{t+1} \|^2_{\Vmat^{\mcal{H}}_{t}} \leq \gamma_t(\delta)\right\}\; .$
    \ENDIF
   \end{spacing}
   \vspace{0pt}
   \ENDFOR
\end{algorithmic}
\end{algorithm*}

%% file: appendix/notations.tex
\newpage
\section{PRELIMINARIES}\label{app:prel}

\subsection{Notations}
\label{app:notations}

We detail below useful notations that will be used throughout the appendix. Below $T\in\mbb{N}^+$, $U$ is a set, $\Theta\subset\mbb{R}^d$ is a compact set, $a\in\mcal{A}$, $r\in\{0,1\}$ and $x,y\in\mbb{R}^d$.

\begin{tabular}{l@{\hskip 30pt}l}
	$[T]$ &the set of integers from $1$ to $T$.\\[3pt]
	$ \vert U\vert$ & cardinality of $U$. \\[3pt]
	$\diam(\Theta) = \max_{\theta_1,\theta_2}\| \theta_1-\theta_2\|$ &diameter of $\Theta$.\\[3pt]
	$\diam_\mcal{A}(\Theta) = \max_{a\in\mcal{A}} \max_{\theta_1,\theta_2} \vert a\transp(\theta_1-\theta_2)\vert$ &diameter of $\Theta$ under $\mcal{A}$.\\[3pt]
	$\mu(x) = (1+\exp(-x))^{-1}$ &the logistic function at $x$.\\[3pt]
	$\ell(x,r) = -r\log\mu(x) - (1-r)\log(1-\mu(x))$ &log-loss associated to $(x,r)$.\\[3pt]
	$\ell_{t+1}(\theta)=\ell(a_t\transp\theta, r_{t+1})$ &instantaneous log-loss of $\theta$ at round $t$.\\[3pt]
	$\bar\ell_{t+1}(\theta)=\ell(a_t\transp\theta, 1-r_{t+1})$ &``reverse'' instantaneous log-loss of $\theta$ at round $t$.\\[3pt]
	$\Hmat_t(\theta) = \sum_{s=1}^t \dot\mu(a_s\transp\theta)a_sa_s\transp + \lambda\mbold{I}_d$ & Hessian of the cumulative-log loss at $\theta$ up to $t$.\\[3pt]
	$\Vmat_t =  \sum_{s=1}^t a_sa_s\transp/\kappa + \lambda\mbold{I}_d$ &linear-like design-matrix up to $t$.
\end{tabular}

Note that for all $\theta$ s.t $\| \theta\|\leq S$ we have $\dot\mu(a_s\transp\theta)\geq 1/\kappa$ by definition of $\kappa$. Therefore:
\begin{align}\label{eq:Ht2Vt}
		\forall\theta\text{ s.t } \|\theta\|\leq S, \quad \mbold{H}_t(\theta) \succeq \mbold{V}_t \; .
\end{align}

Below we define several ``slowly growing'' functions (uniformly denoted $\gamma_t(\delta)$ in the main paper). They will be used throughout the proofs. 
\begin{align}
	\lambda_t(\delta) &= d\log((4+t/4)/\delta)\;,\label{eq:def_lambda} \\
	\gamma_t(\delta) &=  (S+3/2)^2\lambda_t(\delta)\;,\label{eq:def_gamma}\\
	\beta_t(\delta) &=  \left(5/2+(S+3/2)^2 + S\right)^2\gamma_t(\delta)\;,\label{eq:def_beta}\\
	\nu_t(\delta)&=1/2 + 2\log\left(2\sqrt{t/4+1}/\delta\right)\;,\label{eq:def_nu}\\
	\sigma_t(\delta) &= 8S^2 + 6 + 4\log(t) + 9\nu_t(\delta) + 18\exp(1)d\log(1+t/(4d))\;,\label{eq:def_sigma}\\
	\eta_t(\delta) &= 4 + 4\log(t) + 16S^2 + (2+2S)^2\nu_t(\delta)/2 + 8(1+S)d\log(1+t/d)\;.\label{eq:def_eta}
\end{align}

\subsection{Useful Inequalities and Self-Concordant Control}
A central idea when analyzing \logb{} is to tightly link \emph{estimation} errors (\emph{e.g.} between $\theta_1$ and $\theta_2$) to \emph{prediction} errors (\emph{e.g.} between $\mu(a\transp\theta_1)$ and $\mu(a\transp\theta_2)$). Exact Taylor expansion is a powerful tool to achieve this; as for previous works we will use it abundantly and in the following lines we introduce useful notations to this end. Specifically, for any $a\in\mcal{A}$ and $x,y\in\mbb{R}$ define:
\begin{align}
	\alpha(x,y) &= \int_{v=0}^1 \dot\mu(x+v(y-x))dv = \alpha(y,x) \; ,\\
	\widetilde{\alpha}(x,y) &= \int_{v=0}^1 (1-v)\dot\mu(x+v(y-x))dv\; .
\end{align}  
After exact Taylor expansions we have the following identities for all $\theta_1,\theta_2\in\mbb{R}^d$:
\begin{align}
	\mu(a\transp\theta_2) - \mu(a\transp\theta_1) &= \alpha(a\transp\theta_1,a\transp\theta_2)a\transp(\theta_2-\theta_1)\; ,\label{eq:mvt_mu}\\
	\ell_{t+1}(\theta_2) -\ell_{t+1}(\theta_1) &= \nabla\ell_{t+1}(\theta_1)\transp(\theta_2-\theta_1) + \tilde\alpha(a\transp\theta_1,a\transp\theta_2)(a\transp(\theta_2-\theta_1))^2 \label{eq:mvt_ell}\; .
\end{align}

Below are reminded some useful inequalities that stem from the self-concordance property of the logistic function (the fact that $\vert \ddot\mu\vert \leq \dot\mu$) The proofs can all be found in Appendix F of \cite{abeille2020instance}. 
\begin{align}
	\alpha(x,y)&\geq (1+\vert x-y\vert)^{-1}\dot\mu(z) \text{ for } z\in\{x,y\} \label{eq:link_alpha}\;, \\
	\widetilde\alpha(x,y)&\geq (2+\vert x-y\vert)^{-1}\dot\mu(x)\; , \label{eq:link_alpha_tilde}\\
	\dot\mu(x) &\leq \dot\mu(y) \exp\left(\vert x-y\vert\right) \label{eq:link_mu} \; .
\end{align}

%% file: appendix/warmup.tex
\newpage
\section{WARM-UP PROCEDURE}\label{app:warmup}

We recall the warm-up procedure in \cref{alg:warmup2} for which we now we give the exact values for the ``slowly growing'' functions that we use.

\setcounter{procedure}{0}
\begin{procedure}[tb]
   \caption{\texttt{WarmUp} (detailed)}
   \label{alg:warmup2}
	\begin{algorithmic}
	\REQUIRE{length $\tau$.}
	\STATE Set $\lambda\leftarrow \lambda_\tau(\delta)$, initialize $\Vmat_0\leftarrow\lambda \mbold{I}_d$.\hfill \COMMENT{$\lambda_t(\delta)$ is defined in \cref{eq:def_lambda}}
   \FOR{$t\in[1,\tau]$}
   	\STATE Play $a_t\in\argmax_{\mcal{A}} \| a\|_{\Vmat^{-1}_{t-1}}$, observe $r_{t+1}$. 
	\STATE Update $ \Vmat_t \leftarrow \Vmat_{t-1} + a_ta_t\transp/\kappa$.
   \ENDFOR
   \STATE Compute $\hat\theta_{\tau+1} \leftarrow \argmin_\theta \sum_{s=1}^\tau \ell_{s+1}(\theta) + \lambda\|\theta\|^2/2$.
   \ENSURE $\Theta = \left\{ \theta, \left\|\theta-\hat\theta_{\tau+1}\right\|^2_{\Vmat_{\tau}} \leq{\beta_{\tau}(\delta)}\right\}$. \hfill \COMMENT{$\beta_t(\delta)$ is defined in \cref{eq:def_beta}}
   \end{algorithmic}
\end{procedure}

\subsection{Warm-up Length and Parameter Set}
The goal of this section is to prove the claim behind \cref{eq:diam_claim}, tying the length of the warm-up phase to the diameter of the induced parameter set $\Theta$. The formal claim is made explicit in the following proposition. 
 
\begin{restatable}{prop}{propdiameter}\label{prop:diameter_warmup}
Let $\delta\in(0,1]$. Setting $\tau = \const \kappa S^6 d^2 \log(T/\delta)^2$ ensures that $\Theta$ returned by \warmup{\tau} satisfies:
\begin{align*}
	&\textnormal{(1)}\;\;  \mbb{P}(\theta_\star\in\Theta)\geq1-\delta\; ,\\
	&\textnormal{(2)}\;\; \diam_{\mcal{A}}(\Theta) \leq 1\; .
\end{align*}
\end{restatable}

\begin{proof}
The set  $\Theta$ returned by \warmup{\tau} is: 
\begin{align}
	\Theta = \Big\{ \theta, \big\|\theta-\hat\theta_{\tau+1}\big\|^2_{\Vmat_{\tau}} \leq \beta_\tau(\delta)\Big\}\; . \label{eq:def_theta_wu}
\end{align}
where $\beta_t(\delta)$ is defined in \cref{eq:def_beta}. It satisfies $\beta_t(\delta)\leq \const S^6d\log(t/\delta)$.\\
To prove (1) we claim \cref{prop:concentraitonnoproj} which proof is deferred to \cref{app:concentration_refines_faury}. 
 \begin{restatable}{lemma}{propconcentrationnoproj}
 \label{prop:concentraitonnoproj}
 	Let $\delta\in(0,1]$. Then:
	\begin{align*}
		\mbb{P}\Big( \forall t\geq 1, \; \big\| \theta_\star - \hat\theta_{t+1} \big\|^2_{\Hmat_{t}(\theta_\star)} \leq \beta_t(\delta) \Big)\geq 1 - \delta\; .
	\end{align*}
 \end{restatable}
 
The proof of (1) directly follows:
\begin{align*}
	\mbb{P}\left(\theta_\star\in\Theta\right) &= \mbb{P}\left( \big\|\theta_\star-\hat\theta_{\tau+1}\big\|^2_{\Vmat_{\tau}} \leq \beta_\tau(\delta)\right) &(\text{def. of }\Theta)\\
	&\geq \mbb{P}\left( \big\|\theta_\star-\hat\theta_{\tau+1}\big\|^2_{\Hmat_{\tau}(\theta_\star)} \leq \beta_\tau(\delta) \right) &(\Vmat_\tau \preceq  \Hmat_\tau(\theta_\star), \; \text{\cref{eq:Ht2Vt}}) \\
	&\geq 1-\delta \; .&(\text{\cref{prop:concentraitonnoproj}})
\end{align*}
To prove (2) we claim \cref{prop:warmup} which proof is provided in \cref{app:length_warmup}:

\begin{restatable}{lemma}{propwarmup}\label{prop:warmup}
Let $T\in\mbb{N}^+$ and $\tau\in[T]$. Let $\Theta$ the set returned by \textnormal{\texttt{WarmUp}($\tau$)}. Then:
\begin{align*}
	\diam_\mcal{A}(\Theta) \leq  4\sqrt{\frac{\kappa\beta_T(\delta)d\log(1+T)}{\tau}} \; .
\end{align*}
\end{restatable}
Therefore $\tau = 16 \kappa d\beta_T(\delta)\log(1+T)$ ensures that $\diam_\mcal{A}(\Theta)\leq 1$. 
Since $\beta_T(\delta) \leq \const S^6d\log(T/\delta)$ setting:
\begin{align*}
	\tau =  \const \kappa S^6 d^2\log(T/\delta)^2\; , 
\end{align*}
yields $\diam_\mcal{A}(\Theta)\leq 1$ which finishes proving (2). 
\end{proof}

\subsection{Proof of \cref{prop:warmup}}
\label{app:length_warmup}

\propwarmup*
\begin{proof}
The proof is  inspired by the demonstration of Lemma 8 of from \cite{valko14spectral}. Recall:
\begin{align*}
	\Theta = \left\{ \theta, \big\|\theta-\hat\theta_{\tau+1}\big\|^2_{\Vmat_{\tau}} \leq\beta_\tau(\delta)\right\}\; ,
\end{align*}
with $\Vmat_\tau = \sum_{s=1}^\tau a_sa_s\transp/\kappa + \lambda_\tau(\delta)\mbold{I}_d$ and for all $s\leq \tau$:
\begin{align}\label{eq:defactionwarmup}
	a_s\in\argmax_{\mcal{A}} \| a\|_{\Vmat_{s-1}^{-1}}\; .
\end{align}
Therefore:
\begin{align*}
	\diam_{\mcal{A}}(\Theta) &= \max_{a\in\mcal{A}} \max_{\theta_1,\theta_2} \vert a\transp(\theta_1-\theta_2)\vert\; .\\
						&\leq \max_{a\in\mcal{A}} \max_{\theta_1,\theta_2} \|a\|_{\Vmat_{\tau}^{-1}} \| \theta_1-\theta_2\|_{\Vmat_{\tau}} &(\text{Cauchy-Schwarz})\\
						&\leq 2\sqrt{\beta_\tau(\delta)}\max_{a\in\mcal{A}} \|a\|_{\Vmat_{\tau}^{-1}}  &(\theta_1,\theta_2\in\Theta)\\
						&= 2\sqrt{\beta_\tau(\delta)} \sqrt{\max_{a\in\mcal{A}} \|a\|^2_{\Vmat_{\tau}^{-1}}}  &\\%(\theta_1,\theta_2\in\Theta)\\
						&= 2\sqrt{\beta_\tau(\delta)}\tau^{-1/2}\sqrt{\sum_{s=1}^{\tau} \max_{a\in\mcal{A}}\|a\|^2_{\Vmat_{\tau}^{-1}}} \\
						&\leq  2\sqrt{\beta_\tau(\delta)}\tau^{-1/2}\sqrt{\sum_{s=1}^{\tau} \max_{a\in\mcal{A}}  \|a\|^2_{\Vmat_{s-1}^{-1}}} &(\Vmat_\tau \succeq \Vmat_{s-1})\\
						&\leq   2\sqrt{\beta_\tau(\delta)}\tau^{-1/2}\sqrt{\sum_{s=1}^{\tau} \|a_s\|^2_{\Vmat_{s-1}^{-1}}} &(\text{\cref{eq:defactionwarmup}})\\
						&=   2\sqrt{\beta_\tau(\delta)}\tau^{-1/2}\sqrt{\kappa}\sqrt{\sum_{s=1}^{\tau} \|a_s/\sqrt{\kappa}\|^2_{\Vmat_{s-1}^{-1}}} &\\
						&\leq   4\sqrt{\beta_\tau(\delta)}\tau^{-1/2}\sqrt{\kappa}\sqrt{d\log(1+\tau/d)} &(\text{\cref{lemma:ellipticalpotential}})
\end{align*}
which yields the announced result since $T\geq \tau$. Notice the re-normalization of the action by $\kappa$ so that we can apply the Elliptical Potential Lemma (\cref{lemma:ellipticalpotential}) directly. 

\end{proof}

\subsection{ONS and Warm-Up} \label{app:ons_warmup}
The ONS-like approaches of \cite{zhang2016online,jun2017scalable} do not use a warm-up procedure and rely on a ``crude'' parameter set $\Theta_0 = \{\theta,\; \|\theta\|\leq S\}$. As discussed in the main paper (see \cref{sec:ecolog_disc}) it is natural to wonder whether their mechanisms could be directly improved (by order of magnitude $\kappa$) by using the refined parameter $\Theta$ returned by the warm-up procedure. This is unfortunately not the case; both approaches hard-codes the $\kappa$-dependency in the size of their parameter updates. This dependency can only be \emph{marginally} reduced when using $\Theta$.   

For instance \cite{jun2017scalable} rely the exp-concavity constant of the log-loss to design their update rule. Formally, for a parameter set $\Theta'$ it is defined as (see \citet[Definition 4.1]{hazan2016intro}):
\begin{align*}
	\rho(\Theta') \defeq \sup_{r>0} \left \{ r \text{ s.t }\, \nabla_\theta^2\ell(a\transp\theta, r) \succeq r \nabla_\theta\ell(a\transp\theta, r)\nabla_\theta\ell(a\transp\theta, r)\transp,\; \forall \theta\in\Theta', \forall (a,r)\in\mcal{A}\times\{0,1\}\right\}\; .
\end{align*}
After some straight-forward manipulations it writes as: 
\begin{align*}
	\rho(\Theta') &= \sup_{r>0} \left \{ r \text{ s.t }\, r \leq \dot\mu(a\transp\theta)/(\mu(a\transp\theta)-r)^2 ,\; \forall \theta\in\Theta', \forall (a,r)\in\mcal{A}\times\{0,1\}\right\}\\
	&\leq 2 \min_{\theta\in\Theta'} \min_{a\in\mcal{A}} \dot\mu(a\transp\theta)\; .
\end{align*}
The update rule designed by \cite{jun2017scalable} hard-codes a factor $\rho(\Theta_0)^{-1}$ in their update rule and therefore in the radius of the associated confidence regions. This induces exponentially inflated confidence sets as:
\begin{align*}
	\rho(\Theta_0)^{-1} \geq  \kappa/2 = \const \exp(S) \; .
\end{align*}
Refining this dependency by using  a smaller $\Theta$ does not remove such exponential dependencies in problem-dependent constants (\emph{e.g.} $\|\theta_\star\|$, $S$). Indeed if $\Theta$ is the set returned by \texttt{WarmUp}($\tau$) under the conditions of \cref{prop:diameter_warmup}:
\begin{align*}
	\rho(\Theta)^{-1} \geq \const \exp(\|\theta_\star\|) \; .
\end{align*}
A similar argument holds for the update mechanism \cite{zhang2016online}, which rely on the strong-convexity constant of the log-loss.

%% file: appendix/confset.tex
\newpage

\section{CONCENTRATION AND CONFIDENCE SETS}\label{app:cs}

\subsection{Refinement of \cite{faury2020improved}}
\label{app:concentration_refines_faury}
In the following, we consider that we have adaptively collected the dataset $\{a_t,r_{t+1}\}_t$. We denote:
\begin{align*}
	\hat\theta_{t+1} \defeq \argmin_{\theta} \sum_{s=1}^t \ell_{s+1}(\theta) + \lambda_t(\delta)\| \theta\|^2/2\; ,
\end{align*}
where $\lambda_t(\delta)$ is defined in \cref{eq:def_lambda}. Directly following the proof of \citet[Lemma 11]{faury2020improved}:
 \begin{align}\label{eq:concentration_proj_faury}
 	\mbb{P}\left( \forall t\geq 1, \; \big\| \theta_\star - \widetilde\theta_{t+1}\big \|_{\Hmat_{t}(\theta_\star)}^2 \leq 4(1+2S)^2\gamma_t(\delta) \right)\geq 1 - \delta\; ,
 \end{align} 
 where $\widetilde{\theta}_{t+1}$ is obtained by ``projecting'' $\hat\theta_{t+1}$ on the ball $\{\|\theta\|\leq S\}$ through a \emph{non-convex} minimization routine. The slowly growing function $\gamma_t(\delta)$ is obtained after applying simple upper-bounding operations to \citet[Theorem 1]{faury2020improved} and is formally defined in \cref{eq:def_gamma}. It checks:
 \begin{align*}
 	\gamma_t(\delta) \leq \const S^2 d\log(t/\delta)\; .
 \end{align*}
The following proposition establishes that \cref{eq:concentration_proj_faury} still holds when $\tilde\theta_{t+1}$ is replaced by $\hat\theta_{t+1}$, at the price of only a minor degradation of the bound. This essentially removes the need to solve a non-convex program whenever $\|\hat\theta_{t+1}\|\geq S$. The function $\beta_t(\delta)$ is defined in \cref{eq:def_beta} and checks $\beta_t(\delta)\leq \const S^6 d\log(t/\delta)$. 

\propconcentrationnoproj*
 
 \begin{rem}\label{rem:about_proj}
Whenever $\|\hat\theta_{t+1}\|\leq S$ one can directly use the bound given in \cref{eq:concentration_proj_faury}, which is then valid for $\tilde\theta_{t+1}=\hat\theta_{t+1}$. 
\end{rem}

 \begin{proof}
 	The proof leverages the self-concordance property of the logistic function by using some intermediary results from \cite{abeille2020instance}. In the following, we denote for all $\theta$:
	\begin{align*}
		g_t(\theta) \defeq \sum_{s=1}^t \mu(a_s\transp\theta)a_s + \lambda\theta \quad \text{ and } \quad \Gmat_t(\theta) = \sum_{s=1}^t \alpha(a_s\transp\theta, a_s\transp\theta_\star)a_sa_s\transp\; , 
	\end{align*}
	where $\alpha(x,y)$ is defined in \cref{app:notations}. Further, define the event $E_\delta$ as follows:
	\begin{align*}
		E_\delta \defeq \left\{\forall t\geq 1, \; \left\| g_t(\theta_\star) - g_t(\hat\theta_{t+1})\right\|_{\Hmat_{t}(\theta_\star)^{-1}}^2 \leq \gamma_t(\delta)\right\} \; .
	\end{align*}
	By Lemma 1 of \cite{faury2020improved} we have that  $\mbb{P}(E_\delta) \geq 1-\delta$. 
	From the demonstration of Lemma 2 from \cite{abeille2020instance} it can also be extracted that if $E_\delta$ holds then for any $t\geq 1$: 
	\begin{align}
		\Hmat_t(\theta_\star)&\preceq \left(1+\gamma_t(\delta)/\lambda_t(\delta) + \sqrt{\gamma_t(\delta)/\lambda_t(\delta)}\right)\Gmat_t(\hat\theta_{t+1})\notag\\ 
		&=  \left(5/2+(S+3/2)^2 + S\right)\Gmat_t(\hat\theta_{t+1}) \;. \label{eq:fromabeille}
	\end{align}
	Finally, recall that by the mean-value theorem we have the following identity for any $\theta$:
	\begin{align}\label{eq:mvt}
		g_t(\theta) - g_t(\theta_\star) = \Gmat_t(\theta) (\theta-\theta_\star)\; .
	\end{align}
	We conclude by chaining inequalities, assuming that $E_\delta$ holds (which happens with probability at least $1-\delta$);
	\begin{align*}
		\left \| \theta_\star - \hat\theta_{t+1} \right\|^2_{\Hmat_{t}(\theta_\star)} &\leq\left(5/2+(S+3/2)^2 + S\right)\left\| \theta_\star - \hat\theta_{t+1} \right\|^2_{\Gmat_{t}(\hat\theta_{t+1}) } &(\text{\cref{eq:fromabeille}})\\
		 &=\left(5/2+(S+3/2)^2 + S\right)\left\| g_t(\theta_\star) - g_t(\hat\theta_{t+1}) \right\|^2_{\Gmat_{t}(\hat\theta_{t+1})^{-1} } &(\text{\cref{eq:mvt}})\\
		 &\leq \left(5/2+(S+3/2)^2 + S\right)^2 \left\| g_t(\theta_\star) - g_t(\hat\theta_{t+1}) \right\|^2_{\Hmat_{t}(\theta_\star)^{-1} }&(\text{\cref{eq:fromabeille}}) \\
		 &\leq\left(5/2+(S+3/2)^2 + S\right)^2\gamma_t(\delta) =  \beta_t(\delta)\; , &(E_\delta \text{ holds})
	\end{align*}
	which proves the announced result. 
 \end{proof}

 \subsection{Statement of \cref{thm:conf_set}}\label{app:our_conf_set}
 
 The goal of this section is to justify the confidence sets used in the main paper through the statement of the more general \cref{thm:conf_set} (see below). In particular, we deal here with the optimization errors introduced when running the \texttt{ECOLog} procedure.
 
 \setcounter{algorithm}{3}
\begin{algorithm}[h!]
   \caption{\underline{E}ffi\underline{c}ient L\underline{o}cal Learning for \underline{Log}istic Bandits (\texttt{ECOLog}, sequential form)}
   \label{alg:ecolog2}
	\begin{algorithmic}
	\REQUIRE{Compact convex sets $\{\Theta_t\}_t$, optimization accuracies $\{\eps_t\}_t$.}
	\STATE Let $\Wmat_1 \leftarrow \mbold{I}_d$, $\theta'_1\in\Theta_1$. \hfill\COMMENT{initialization} 
	\STATE Let $D \leftarrow \sup_{t\geq 1} \diam_\mcal{A}(\Theta_t)$.
   \FOR{$t\geq 1$}
   \STATE Receive the pair $(a_t, r_{t+1})$.
   \STATE Define $\theta_{t+1}$ as:
   $$
   	\theta_{t+1} = \argmin_{\theta\in\Theta_t} \left(\frac{1}{2+D}\elltwo{\theta-\theta'_t}^2_{\Wmat_t} + \ell_{t+1}(\theta)\right)\; .
   $$
   \STATE Compute $\theta'_{t+1}$ by solving the above program to accuracy $\eps_t$.
   \STATE Update $\Wmat_{t+1} \leftarrow \Wmat_t + \dot\mu(a_t\transp\theta'_{t+1})a_ta_t\transp$.
   \ENDFOR
\end{algorithmic}
\end{algorithm}

We detail in \cref{alg:ecolog2} the pseudo-code for \texttt{ECOLog} in its sequential form. It takes as input a sequence of compact convex sets $\{\Theta_t\}_t$ and a sequence $\{\eps_t\}_t$ of optimization accuracy. Note the use of:
\begin{align*}
	D \defeq \sup_{t\geq 1} \diam_\mcal{A}(\Theta_t)
\end{align*}
If this quantity is unknown, $D$ is replaced by an upper-bound on the supremum (the tighter, the better). Our use of \texttt{ECOLog} in both \cref{alg:ofuecolog,alg:dilutedofuecolog} falls under this general description. For instance \cref{alg:ofuecolog} instantiates this procedure with $\Theta_t \equiv \Theta$ (the set returned by the warm-up) for which $D\leq1$ (see \cref{prop:diameter_warmup}). 

We assume that at each round $t\geq 1$ the true minimizer:
\begin{align}
	\theta_{t+1} = \argmin_{\theta\in\Theta_t} \left(\frac{1}{2+D}\elltwo{\theta-\theta_t}^2_{\Wmat_t} + \ell_{t+1}(\theta)\right)\label{eq:def_thetap}\; ,
\end{align}
can computed up to accuracy $\eps_t$. In other words, we have access to $\theta'_{t+1}$ such that:
\begin{align}
	\left\| \theta_{t+1}-\theta'_{t+1}\right\|  \leq \eps_t\; .\label{eq:eps_approx}
\end{align}
We discuss in \cref{app:cost} how such $\theta'_{t+1}$ can be efficiently computed. 
We denote $\{(\theta'_{t+1},\Wmat_{t+1})\}_t$ the sequence of parameters maintained by \ecolog{\{\Theta_t\}_t,\,\{\eps_t\}_t} and claim the following concentration bound. The function $\nu_t(\delta)$ is defined in \cref{eq:def_nu}. Numerical constants can be improved by a more careful analysis. 

\begin{restatable}{thm}{thmconfset}\label{thm:conf_set}
	Let $\delta\in(0,1]$ and assume that $\theta_\star \in \Theta_t$ for all $t\geq 1$. Then:
	\begin{align*}
		\mbb{P}\left(\forall t\geq 1, \; \elltwo{\theta'_{t+1}-\theta_\star}_{\Wmat_{t+1}}^2 \leq  8S^2 + 4\sum_{s=1}^{t} s\varepsilon_s^2 + 2D^2  + (2+D)^2\nu_t(\delta)/2 + 2(2+D)^2\exp(D)d\log(1+t/(4d))\right)\\\geq 1-\delta\; .
	\end{align*}
\end{restatable}

\subsection{Proof of \cref{thm:conf_set}}

An important technical piece of our analysis resides in the following Lemma, which derives a \emph{local} quadratic lower-bound for the logistic loss $\ell_{t+1}(\theta)$. It is extracted from the self-concordance analysis of \cite{abeille2020instance}. A slightly stronger form, derived through other means, also appears in \cite{jezequel2020efficient}. The proof is deferred to \cref{app:proofquadlb}.

\begin{restatable}[Local Quadratic Lower-Bound]{prop}{lemmaquadlowerbound}
\label{lemma:local_lb}
For all $t\geq 1$ and any $\theta,\theta_r\in\Theta_t$:
\begin{align*}
	\ell_{t+1}(\theta) \geq \ell_{t+1}(\theta_r) + \nabla\ell_{t+1}(\theta_r)\transp(\theta-\theta_r) + \frac{\dot\mu(a_t\transp\theta_r)}{2+\diam_\mcal{A}(\Theta_t)}(a_t\transp(\theta-\theta_r))^2\; .
\end{align*}
\end{restatable}

Another important intermediary result is given by the following Lemma. It is obtained by directly leveraging the update rule. The proof is deferred to \cref{app:kktapplied}. 

\begin{restatable}{lemma}{propkktapplied}\label{prop:kktapplied}
At any round $t\geq 1$:
\begin{align*}
	\nabla\ell_{t+1}(\theta_{t+1})\transp(\theta_{t+1}-\theta_\star) \leq (1+D/2)^{-1}(\theta_{t+1}-\theta'_t)\transp\Wmat_t(\theta_\star-\theta_{t+1}) \; .
\end{align*}
\end{restatable}

Combining \cref{lemma:local_lb,prop:kktapplied} yields the following result, tying the deviation between $\theta'_{t+1}$ and $\theta_\star$ with the excess loss incurred by  $\{\theta_{s+1}\}_{s=1}^t$. The proof is deferred to \cref{app:lemmadecomp}.

\begin{restatable}{lemma}{lemmadecomp}\label{lemma:decomp}
	For any $t\geq 1$ the following holds:
	\begin{align*}
		\elltwo{\theta'_{t+1}-\theta_\star}_{\Wmat_{t+1}}^2 \leq  4S^2 + 4\sum_{s=1}^{t} s\eps_s^2 + (4+2D)\left[\sum_{s=1}^{t} \ell_{s+1}(\theta_\star) - \ell_{s+1}(\theta_{s+1})\right] \; .
	\end{align*}
\end{restatable}

To obtain a valid confidence set from \cref{lemma:decomp} we are left to bound $\sum_{s=1}^{t} \ell_{s+1}(\theta_\star)-\ell_{s+1}(\theta_{s+1})$. To do so, and inspired by the analysis of \cite{jezequel2020efficient} in the online convex optimization setting, we introduce:
\begin{align}\label{eq:def_theta_bar}
	\bar\theta_t \defeq \argmin_{\Theta} \left(\frac{1}{2+D}\elltwo{\theta-\theta_t}_{\Wmat_t}^2 + \ell(a_t\transp\theta,0)+\ell(a_t\transp\theta,1)\right)\; .
\end{align}
Note that $\bar\theta_t$ is $\mcal{F}_t$-measurable ($\theta_{t+1}$ is $\mcal{F}_{t+1}$ measurable). We rely on the following decomposition and bound each term of the r.h.s separately:
\begin{align}
	\sum_{s=1}^{t} \ell_{s+1}(\theta_\star)-\ell_{s+1}(\theta_{s+1}) = \left[\sum_{s=1}^{t} \ell_{s+1}(\theta_\star)-\ell_{s+1}(\bar\theta_{s})\right] + \left[\sum_{s=1}^{t} \ell_{s+1}(\bar\theta_s)-\ell_{s+1}(\theta_{s+1})\right]\; .
	\label{eq:loss_decomp}
\end{align}

The first term is bounded with high probability as stated below. The proof is deferred to \cref{app:concentration_loss} and uses a 1-dimensional version of a concentration result from \cite{faury2020improved}. 
\begin{restatable}{lemma}{lemmaconcentrationloss}\label{lemma:concentration_loss}
Let $\delta\in(0,1]$. We have: 
\begin{align*}
	\mbb{P}\left(\forall t\geq 1, \; \sum_{s=1}^t \ell_{s+1}(\theta_\star)-\ell_{s+1}(\bar\theta_s) \leq (2+D)\nu_t(\delta)/4 + D^2(2+D)^{-1}\right)\geq 1-\delta'\; .
\end{align*}
\end{restatable} 

We now turn on bounding the second term in \cref{eq:loss_decomp} - that is:
\begin{align*}
	\sum_{s=1}^{t} \ell_{s+1}(\bar\theta_s)-\ell_{s+1}(\theta_{s+1}) \; .
\end{align*}
We claim the following intermediary result, which proof is deferred to \cref{app:last_term_int}.
\begin{restatable}{lemma}{lemmalasttermint}\label{lemma:last_term_int}
	The following result holds for any $t\geq 1$:
	\begin{align*}
		\sum_{s=1}^{t}\ell_{s+1}(\bar\theta_s)-\ell_{s+1}(\theta_{s+1}) \leq (1+D)\sum_{s=1}^t \dot\mu(a_s\transp\bar\theta_s)\elltwo{a_s}^2_{\Wmat_{s+1}^{-1}}\; .
	\end{align*}
\end{restatable}

We finish the bound by the following result, a consequence of the self-concordance property of the logistic function. The proof is deferred to \cref{app:last_term_end}.  

\begin{restatable}{lemma}{lemmalasttermend}\label{lemma:last_term_end}
	The following result holds for any $t\geq 1$:
	\begin{align*}
		\sum_{s=1}^t \dot\mu(a_s\transp\bar\theta_s)\elltwo{a_s}^2_{\Wmat_{s+1}^{-1}} \leq \exp(D) d\log\left((1+t/(4d)\right) \; .
	\end{align*}
\end{restatable}

Combining \cref{lemma:last_term_int,lemma:last_term_end} yields that:
\begin{align*}
	\sum_{s=1}^{t}\ell_{s+1}(\bar\theta_s)-\ell_{s+1}(\theta_{s+1}) \leq (1+D)\exp(D) d\log\left((1+t/(4d)\right)\; .
\end{align*}
Assembling this result with \cref{eq:loss_decomp} yields that $\forall t\geq 1$:
\begin{align*}
	  \sum_{s=1}^{t} \ell_{s+1}(\theta_\star)-\ell_{s+1}(\theta_{s+1}) \leq \left[\sum_{s=1}^{t} \ell_{s+1}(\theta_\star)-\ell_{s+1}(\bar\theta_{s})\right] + (1+D)\exp(D)  d\log\left((1+t/(4d)\right)\; .
\end{align*}
Thanks to \cref{lemma:concentration_loss} this further yields that with probability at least $1-\delta$:
\begin{align*}
	\forall t\geq 1, \;    \sum_{s=1}^{t} \ell_{s+1}(\theta_\star)-\ell_{s+1}(\theta_{s+1}) \leq (2+D)\nu_t(\delta)/4 + D^2(2+D)^{-1} +  (1+D)\exp(D) d\log\left((1+t/(4d)\right)\; .
\end{align*}
Assembling this result with \cref{lemma:decomp} along with simple bounding operations yield the announced result.

\subsubsection{Proof of \cref{lemma:local_lb}}\label{app:proofquadlb}
\lemmaquadlowerbound*

\begin{proof}
	By a exact second-order Taylor of $\ell_{t+1}(\theta)$ decomposition around $\theta_r$ yields (see \cref{eq:mvt_ell}):
	\begin{align*}
		\ell_{t+1}(\theta) = \ell_{t+1}(\theta_r) + \nabla\ell_{t+1}(\theta_r)\transp(\theta-\theta_r)+ \widetilde\alpha(a_t\transp\theta, a_t\transp\theta_r)(a_t\transp(\theta-\theta_r))^2\; ,
	\end{align*}
	Further by \cref{eq:link_alpha_tilde} we have:
	\begin{align*}
		 \widetilde\alpha(a_t\transp\theta, a_t\transp\theta_r) &\geq \dot{\mu}(a_t\transp\theta_r)/(2+\vert a_t\transp(\theta_r-\theta)\vert)&\, \\
		 &\geq  \dot{\mu}(a_t\transp\theta_r)/(2+\diam_\mcal{A}(\Theta_t))\; , &(\theta_r,\theta\in\Theta_t, \, \text{def. of }\diam_\mcal{A}(\Theta_t))
	\end{align*}
	which concludes the proof. 
\end{proof}

\subsubsection{Proof of \cref{prop:kktapplied}}\label{app:kktapplied}
\propkktapplied*

\begin{proof}
	Denote $\widetilde L_{t+1}(\theta) :=(2+D)^{-1}\elltwo{\theta-\theta'_t}^2_{\Wmat_t} + \ell_{t+1}(\theta)$ the function minimized by $\theta_{t+1}$ over $\Theta_t$. Since $\Theta_t$ is a convex set and $\widetilde L_{t+1}$ a convex function, we have that for any $\theta\in\Theta_t$ (see \cref{fact:kkt}); 
	\begin{align*}
		0 &\leq \nabla\widetilde L_{t+1}(\theta_{t+1})\transp(\theta-\theta_{t+1})\\
		&= \left((1+D/2)^{-1}\Wmat_t(\theta_{t+1}-\theta'_t) + \nabla \ell_{t+1}(\theta_{t+1})\right)\transp(\theta-\theta_{t+1})\\
		&= (1+D/2)^{-1}(\theta_{t+1}-\theta'_t)\transp\Wmat_t(\theta-\theta_{t+1}) + \nabla \ell_{t+1}(\theta_{t+1})\transp(\theta-\theta_{t+1})
	\end{align*}
Taking $\theta=\theta_\star\in\Theta_t$ (by assumption) in the above inequality yields the announced result. 
\end{proof}

\subsubsection{Proof of \cref{lemma:decomp}}\label{app:lemmadecomp}
\lemmadecomp*

\begin{proof}
	By \cref{lemma:local_lb}, and because $\theta_{t+1},\theta_\star\in\Theta_t$ (by construction for $\theta_{t+1}$ and by assumption for $\theta_\star$) the following holds for any $s\geq 1$:
	\begin{align*}
		\ell_{s+1}(\theta_\star) &\geq \ell_{s+1}(\theta_{s+1}) + \nabla\ell_{s+1}(\theta_{s+1})\transp(\theta_\star-\theta_{s+1})  + \frac{\mu(a_s\transp\theta_{s+1})}{2+\diam_\mcal{A}(\Theta_t)}(a_s\transp(\theta_\star-\theta_{s+1}))^2\\
		 &\geq \ell_{s+1}(\theta_{s+1}) + \nabla\ell_{s+1}(\theta_{s+1})\transp(\theta_\star-\theta_{s+1})  + \frac{\mu(a_s\transp\theta_{s+1})}{2+D}(a_s\transp(\theta_\star-\theta_{s+1}))^2\; . &(D\geq \diam_\mcal{A}(\Theta))
\end{align*}		
After re-arranging this yields:
\begin{align*}
	\ell_{s+1}(\theta_{s+1}) - \ell_{s+1}(\theta_\star) \leq \nabla\ell_{s+1}(\theta_{s+1})\transp(\theta_{s+1}-\theta_\star) - (2+D)^{-1} \dot{\mu}(a_s\transp\theta_{s+1})(a_s\transp(\theta_{s+1}-\theta_\star))^2\; .
\end{align*}
Using \cref{prop:kktapplied} in the above inequality gives:
\begin{align*}
	(1+D/2)\left(\ell_{s+1}(\theta_{s+1}) - \ell_{s+1}(\theta_\star)\right) &\leq (\theta_{s+1}-\theta'_s)\transp \Wmat_s(\theta_\star-\theta_{s+1}) -\frac{1}{2}\dot{\mu}(a_s\transp\theta_{s+1})(a_s\transp(\theta_{s+1}-\theta_\star))^2\; ,\\
	&= -\frac{1}{2}\elltwo{\theta_{s+1}-\theta_\star}_{\Wmat_{s}}^2 + \frac{1}{2}\elltwo{\theta'_{s}-\theta_\star}_{\Wmat_{s}}^2 - \frac{1}{2}\elltwo{\theta_{s+1}-\theta'_s}_{\Wmat_s}^2-\frac{1}{2}\dot{\mu}(a_s\transp\theta_{s+1})(a_s\transp(\theta_{s+1}-\theta_\star))^2\; ,\\
	&=  -\frac{1}{2}\elltwo{\theta_{s+1}-\theta_\star}_{\Wmat_{s+1}}^2 + \frac{1}{2}\elltwo{\theta'_s-\theta_\star}_{\Wmat_{s}}^2 - \frac{1}{2}\elltwo{\theta_{s+1}-\theta'_s}_{\Wmat_s}^2 \\
	&\leq -\frac{1}{2}\elltwo{\theta_{s+1}-\theta_\star}_{\Wmat_{s+1}}^2 + \frac{1}{2}\elltwo{\theta'_s-\theta_\star}_{\Wmat_{s}}^2\\
	&\leq -\frac{1}{2}\elltwo{\theta_{s+1}-\theta_\star}_{\Wmat_{s+1}}^2 + \frac{1}{2}\elltwo{\theta_{s}-\theta_\star}_{\Wmat_{s}}^2 + \frac{1}{2}\elltwo{\theta_s-\theta'_s}_{\Wmat_{s}}^2 \\
	&\leq  -\frac{1}{2}\elltwo{\theta_{s+1}-\theta_\star}_{\Wmat_{s+1}}^2 + \frac{1}{2}\elltwo{\theta'_{s}-\theta_\star}_{\Wmat_{s}}^2 + s\eps_{s-1}^2\, 
\end{align*}
since $\elltwo{\theta'_{s}-\theta_\star}_{\Wmat_{s}}^2 \leq \lambda_\text{max}(\Wmat_s)  \elltwo{\theta'_{s}-\theta_s}^2 \leq 2s\eps_{s-1}^2$. 
By re-arranging:
\begin{align*}
	(2+D)\left( \ell_{s+1}(\theta_\star)-\ell_{s+1}(\theta_{s+1})\right)-\elltwo{\theta_{s+1}-\theta_s}_{\Wmat_s}^2
 \geq \elltwo{\theta_{s+1}-\theta_\star}_{\Wmat_{s+1}}^2 - \elltwo{\theta_{s}-\theta_\star}_{\Wmat_{s}}^2 - 2s\eps_{s-1}^2
 \end{align*}
and summing from $s=1$ to $t$:
\begin{align*}
	 (2+D)\sum_{s=1}^t \ell_{s+1}(\theta_\star)-\ell_{s+1}(\theta_{s+1}) &\geq \sum_{s=1}^{t} \left[\elltwo{\theta_{s+1}-\theta_\star}_{\Wmat_{s+1}}^2 - \elltwo{\theta_{s}-\theta_\star}_{\Wmat_{s}}^2\right]  - 2\sum_{s=1}^t s\eps_{s-1}^2\\
	 &=  \elltwo{\theta_{t+1}-\theta_\star}_{\Wmat_{t+1}}^2 -  \elltwo{\theta_{1}-\theta_\star}_{\Wmat_{1}}^2 - 2\sum_{s=1}^t s\eps_{s-1}^2&\text{(telescopic sum)}\\
	 &=  \elltwo{\theta_{t+1}-\theta_\star}_{\Wmat_{t+1}}^2 - \elltwo{\theta_1 - \theta_\star}^2- 2\sum_{s=1}^t s\eps_{s-1}^2 &(\Wmat_{1}=\mbold{I}_d)
\end{align*}
After re-arranging and setting $\eps_0=0$ (there is no program to solve at $t=0$); 
\begin{align*}
	  \elltwo{\theta_{t+1}-\theta_\star}_{\Wmat_{t+1}}^2 \leq 4S^2 + 2\sum_{s=1}^{t-1} s\eps_{s}^2 +  (2+D)\left[\sum_{s=1}^t \ell_{s+1}(\theta_\star)-\ell_{s+1}(\theta_{s+1})\right]\; .
\end{align*}
This concludes the proof as:
\begin{align*}
	 \elltwo{\theta'_{t+1}-\theta_\star}_{\Wmat_{t+1}}^2 &\leq  2\elltwo{\theta_{t+1}-\theta_\star}_{\Wmat_{t+1}}^2 + 2\elltwo{\theta'_{t+1}-\theta_{t+1}}_{\Wmat_{t+1}}^2 &((a+b)^2\leq 2(a^2+b^2)\\
	  &\leq 2 \elltwo{\theta_{t+1}-\theta_\star}_{\Wmat_{t+1}}^2 + 2(1+t)\elltwo{\theta'_{t+1}-\theta_{t+1}}^2 &(\Wmat_{t+1}\preceq (1+t)\mbold{I}_d)\\
	  &\leq 2  \elltwo{\theta'_{t+1}-\theta_\star}_{\Wmat_{t+1}}^2 + 4t\eps_t^2 \; . 
\end{align*}
\end{proof}

\subsubsection{Proof of \cref{lemma:concentration_loss}}\label{app:concentration_loss}

\lemmaconcentrationloss* 

\begin{proof}
	Using \cref{lemma:local_lb} with $\theta=\theta_\star$ and $\theta_r = \bar\theta_s$ yields:
	\begin{align}
		\sum_{s=1}^t \ell_{s+1}(\theta_\star)-\ell_{s+1}(\bar\theta_s) &\leq \sum_{s=1}^t \nabla\ell_{s+1}(\theta_\star)\transp(\theta_\star-\bar\theta_s) - (2+D)^{-1}\sum_{s=1}^t \dot\mu(a_s\transp\theta_\star)(a_s\transp(\theta_\star-\bar\theta_{s}))^2 \notag \\
		&=  \sum_{s=1}^t(\mu(a_s\transp\theta_\star)-r_{s+1})a_s\transp(\theta_\star-\bar\theta_s) - (2+D)^{-1}\sum_{s=1}^t \dot\mu(a_s\transp\theta_\star)(a_s\transp(\theta_\star-\bar\theta_{s}))^2 \notag \\
		&=  D\sum_{s=1}^t \eta_{s+1} x_s - D^2(2+D)^{-1}X_t\; ,\label{eq:stoc_int}
	\end{align}
	where we denoted $x_s \defeq a_s\transp(\theta_\star-\bar\theta_s)/D$, $X_t \defeq \sum_{s=1}^t \dot\mu(a_s\transp\theta_\star)x_s^2$ and $\eta_{s+1}\defeq \mu(a_s\transp\theta_\star)-r_{s+1}$. 
We use a 1-dimensional version of the concentration result provided by Theorem 1 of \cite{faury2020improved} to bound $\sum_{s=1}^t \eta_{s+1} x_s$. We remind its general form below for the sake of completeness.

\begin{thm}[Theorem 1 of \cite{faury2020improved}]
	Let $\{\mcal{F}_t\}_{t=1}^\infty$ be a filtration. Let $\{x_t\}_{t=1}^{\infty}$ be a stochastic process in $\mcal{B}_2(d)$ such that $x_t$ is $\mcal{F}_{t}$-measurable. Let $\{\eta_{t}\}_{t=2}^\infty$ be a martingale difference sequence such that $\eta_{t+1}$ is $\mcal{F}_{t+1}$ measurable. Furthermore, assume that conditionally on $\mcal{F}_t$ we have $ \vert \eta_{t+1}\vert \leq 1$ almost surely, and note $\sigma_t^2 \defeq \mbb{E}\left[\eta_{t+1}^2\vert \mcal{F}_t \right]$. Let $\lambda>0$ and for any $t\geq 1$ define:
\begin{align*}
\mbold{H}_t\defeq\sum_{s=1}^{t}\sigma_s^2 x_sx_s^T + \lambda\mbold{I}_d, \qquad S_{t+1}\defeq \sum_{s=1}^{t} \eta_{s+1}x_s.
\end{align*}
Then for any $\delta\in(0,1]$:
\begin{align*}
\mbb{P}\Bigg(\exists t\geq 1, \, \left\lVert S_{t+1}\right\rVert_{\mbold{H}_t^{-1}} \!\geq\! \frac{\sqrt{\lambda}}{2}\!+\!\frac{2}{\sqrt{\lambda}}\log\!\left(\frac{\det\left(\mbold{H_t}\right)^{\frac{1}{2}}\!\lambda^{-\frac{d}{2}}}{\delta}\right)+\frac{2}{\sqrt{\lambda}}d\log(2)\Bigg)\leq \delta.
\end{align*}
\label{thm:concentration_general}
\end{thm}

Recall that we use the filtration $\mcal{F}_{t} \defeq \sigma\left(a_1,r_2,\ldots, a_t\right)$. In our case, $x_s$ is 1-dimensional, is $\mcal{F}_s$-measurable and satisfies $\vert x_s\vert \leq 1$ almost surely (by definition of $D$, and since both $\theta_\star, \bar\theta_s\in\Theta_t$). Further, $\eta_{s+1}$ is $\mcal{F}_{s+1}$-measurable and thanks to \cref{eq:model} we have:
	\begin{align*}
		\mbb{E}\left[ \eta_{s+1}\vert \mcal{F}_s\right] = 0 \quad \text{ and }\mbb{E}\left[ \eta_{s+1}^2\vert \mcal{F}_s\right] = \dot\mu(a_s\transp\theta_\star)\; .
	\end{align*}
Furthermore, note that with the notations of \cref{thm:concentration_general} we have $\mbold{H}_t = X_t + 1$.
By a direct application of \cref{thm:concentration_general} we obtain that with probability at least $1-\delta$:
\begin{align*}
		\forall t\geq 1, \quad \sum_{s=1}^t \eta_{s+1} x_s &\leq \sqrt{X_t + 1}\sqrt{1/2 +2\log\left(\frac{2\sqrt{X_t+1}}{\delta}\right)}&\\
		&=  \sqrt{X_t + \lambda}\sqrt{1/2 + 2\log\left(\frac{2\sqrt{\sum_{s=1}^t \dot\mu(a_s\transp\theta_\star)z_s^2+1}}{\delta}\right)}&\\
		&\leq   \sqrt{X_t + 1}\sqrt{1/2 + 2\log\left(\frac{2\sqrt{t/4+1}}{\delta}\right)} & (\dot\mu \leq 1/4, \vert z_s\vert\leq 1)\\
		&=\sqrt{\nu_t(\delta) }\sqrt{X_t + 1}  &(\text{def. of } \nu_t(\delta))\\
		&\leq \frac{\nu_t(\delta)}{4D(2+D)^{-1}} + D(2+D)^{-1}(X_t+1)
\end{align*}
	where in the second to last inequality we used the fact that $\forall a,b,\zeta>0$ we have $\sqrt{ab}\leq a/(2\zeta) + \zeta b/2$ (applied with $a=\gamma_t(\delta)$, $b=X_t+\lambda$ and $\zeta = 2D(2+D)^{-1}$. Re-injecting in \cref{eq:stoc_int} yields that with probability at least $1-\delta$:
\begin{align*}
		\forall t\geq 1, \quad \sum_{s=1}^{t} \ell_{s+1}(\theta_\star)-\ell_{s+1}(\bar\theta_{s}) \leq (2+D)\nu_t(\delta)/4 + D^2(2+D)^{-1}\; .
\end{align*}
\end{proof}

\subsubsection{Proof of \cref{lemma:last_term_int}}\label{app:last_term_int}
\lemmalasttermint*

\begin{proof}
	By convexity of $\ell_{s+1}(\cdot)$ one has that for all $s\geq 1$:
	\begin{align}
		\ell_{s+1}(\bar\theta_s)-\ell_{s+1}(\theta_{s+1}) &\leq \nabla\ell_{s+1}(\bar\theta_s)\transp(\bar\theta_s-\theta_{s+1})\notag \\
		&\leq  \elltwo{\nabla\ell_{s+1}(\bar\theta_s)}_{\Wmat_{s+1}^{-1}}\elltwo{\bar\theta_s-\theta_{s+1}}_{\Wmat_{s+1}} &\text{(Cauchy-Schwarz)}\label{eq:almost_done}
	\end{align}
	Since $\ell(a_s\transp\theta,0)+ \ell(a_s\transp\theta, 1) = \ell_{s+1}(\theta)+\bar\ell_{s+1}(\theta)$, one can re-write the computation of $\bar\theta_s$ as:
\begin{align*}
	\bar\theta_s = \argmin_{\theta\in\Theta_s}\frac{1}{2+D}\elltwo{\theta-\theta_s}_{\Wmat_t}^2 +\ell_{s+1}(\theta)+\bar\ell_{s+1}(\theta)\; .
\end{align*}
	  By convexity of the objective function minimized by $\bar\theta_s$ and convexity of $\Theta_s$ we therefore have the following inequality (see \cref{fact:kkt}) for any $s\geq 1$:
	\begin{align*}
		(1+D/2)^{-1}(\bar\theta_s-\theta_s)\transp\Wmat_s(\theta_{s+1}-\bar\theta_s) + \nabla\ell_{s+1}(\bar\theta_s)\transp(\theta_{s+1}-\bar\theta_s)+ \nabla\bar\ell_{s+1}(\bar\theta_s)\transp(\theta_{s+1}-\bar\theta_s) \geq 0\; .
	\end{align*}
	since $\theta_{s+1}\in\Theta_s$ by definition. 
	By re-arranging this yields:
	\begin{align*}
		(1+D/2) \nabla\bar\ell_{s+1}(\bar\theta_s)\transp(\theta_{s+1}-\bar\theta_s) &\geq (\bar\theta_s-\theta_s)\transp\Wmat_s(\bar\theta_s-\theta_{s+1}) + (1+D/2)\nabla\ell_{s+1}(\bar\theta_s)\transp(\bar\theta_s-\theta_{s+1}) \\
		 &= \elltwo{\bar\theta_s-\theta_{s+1}}_{\Wmat_s}^2 +  (\theta_{s+1}-\theta_s)\transp\Wmat_s(\bar\theta_s-\theta_{s+1})+(1+D/2)\nabla\ell_{s+1}(\bar\theta_s)\transp(\bar\theta_s-\theta_{s+1})\; .
	\end{align*}
By the same argument, since $\bar\theta_s\in\Theta_s$ we also have the inequality:
\begin{align*}
	(\theta_{s+1}-\theta_s)\transp\Wmat_s(\bar\theta_s-\theta_{s+1}) \geq  (1+D/2)\nabla\ell_{s+1}(\theta_{s+1})\transp(\theta_{s+1}-\bar\theta_s)\; .
\end{align*}
Re-injecting above this yields that:
	\begin{align*}
		 (1+D/2)\nabla\bar\ell_{s+1}(\bar\theta_s)\transp(\theta_{s+1}-\bar\theta_s) &\geq   \elltwo{\bar\theta_s-\theta_{s+1}}_{\Wmat_s}^2 + (1+D/2)(\bar\theta_s-\theta_{s+1})\transp(\nabla\ell_{s+1}(\bar\theta_s)-\nabla\ell_{s+1}(\theta_{s+1}))\\
		 &=   \elltwo{\bar\theta_s-\theta_{s+1}}_{\Wmat_s}^2 + (1+D/2)(\mu(a_s\transp\bar\theta_s)-\mu( a_s\transp\theta_{s+1}))a_s\transp(\bar\theta_s-\theta_{s+1})\\
		 &=   \elltwo{\bar\theta_s-\theta_{s+1}}_{\Wmat_s}^2 + (1+D/2)\alpha(a_s\transp\bar\theta_s, a_s\transp\theta_{s+1})(a_s\transp(\bar\theta_s-\theta_{s+1}))^2\\
		 &\geq \elltwo{\bar\theta_s-\theta_{s+1}}_{\Wmat_s}^2 + (1+D/2)(1+D)^{-1}\dot\mu(a_s\transp\theta_{s+1})(a_s\transp(\bar\theta_s-\theta_{s+1}))^2
	\end{align*}
where in the second to last inequality we used \cref{eq:link_alpha} to obtain $\alpha(a_s\transp\bar\theta_s, a_s\transp\theta_{s+1}) \geq (1+D)^{-1}\dot\mu(a_s\transp\theta_{s+1})$ (since $\theta_{s+1},\bar\theta_s\in\Theta_s$). 
After easy manipulations this yields:
\begin{align*}
	 \elltwo{\bar\theta_s-\theta_{s+1}}_{\Wmat_{s+1} }^2 &\leq  (1+D)\nabla\bar\ell_{s+1}(\bar\theta_s)\transp(\theta_{s+1}-\bar\theta_s)\\
	 &\leq (1+D) \elltwo{\nabla\bar\ell_{s+1}(\bar\theta_s)}_{\Wmat_{s+1}^{-1}}\elltwo{\bar\theta_s-\theta_{s+1}}_{\Wmat_{s+1}} 
\end{align*}
and therefore we obtain that $ \elltwo{\bar\theta_s-\theta_{s+1}}_{\mbold{\widetilde{V}}_{s+1} }\leq (1+D) \elltwo{\nabla\bar\ell_{s+1}(\bar\theta_s)}_{\Wmat_{s+1}^{-1}}$. Assembling with \cref{eq:almost_done};
\begin{align*}
	\ell_{s+1}(\bar\theta_s)-\ell_{s+1}(\theta_{s+1}) &\leq (1+D)\elltwo{\nabla\ell_{s+1}(\bar\theta_s)}_{\Wmat_{s+1}^{-1}}\elltwo{\nabla\bar\ell_{s+1}(\bar\theta_s)}_{\Wmat_{s+1}^{-1}}\\
	&= (1+D)\vert \mu(a_s\transp\bar\theta_s)-r_{s+1}\vert\vert\mu(a_s\transp\bar\theta_s)-1+r_{s+1}\vert \elltwo{a_s}^2_{\Wmat_{s+1}^{-1}}\\
	&= (1+D)\vert \mu(a_s\transp\bar\theta_s)\vert \vert\mu(a_s\transp\bar\theta_s)-1\vert \elltwo{a_s}^2_{\Wmat_{s+1}^{-1}} &(r_{s+1}\in\{0,1\})\\
	&= (1+D)\dot\mu(a_s\transp\bar\theta_s)\elltwo{a_s}^2_{\Wmat_{s+1}^{-1}} &(\mu(1-\mu)=\dot\mu)
\end{align*}
Summing yields the announced result. 
\end{proof}

\subsubsection{Proof of \cref{lemma:last_term_end}}\label{app:last_term_end}
\lemmalasttermend*

\begin{proof}
	By \cref{eq:link_mu}, for all $s\geq 1$:
	\begin{align*}
		\dot\mu(a_s\transp\bar\theta_s) &\leq \exp\left({\vert a_s\transp(\theta'_{s+1}-\bar\theta_s)\vert} \right)\dot\mu(a_s\transp\theta_{s+1})\\
		&\leq \exp(D) \dot\mu(a_s\transp\theta'_{s+1}) \; .&(\theta_{s+1},\bar\theta_s\in\Theta_s, \, D\geq\diam_\mcal{A}(\Theta_s))
	\end{align*}
Denoting $x_s = \sqrt{\mu(a_s\transp\theta'_{s+1})}a_s$ and $\Mmat_{t+1} = \sum_{s=1}^t x_sx_s\transp$, we have:
\begin{align*}
		\sum_{s=1}^t \dot\mu(a_s\transp\bar\theta_s)\elltwo{a_s}^2_{\Wmat_{s+1}^{-1}} &\leq \exp(D)\dot \sum_{s=1}^t \mu(a_s\transp\theta'_{s+1} ) \elltwo{a_s}_{\Wmat_{s+1}^{-1}}^2 \\
		&\leq \exp(D) \sum_{s=1}^t  \elltwo{x_s}_{\Mmat_{s+1}^{-1}}^2 \\
		&= \exp(D) \sum_{s=1}^t  \text{Tr}(\Mmat_{s+1}^{-1} x_sx_s\transp)\\
		&=  \exp(D) \sum_{s=1}^t  \text{Tr}(\Mmat_{s+1}^{-1}( \Mmat_{s+1} - \Mmat_{s}))\\
		&\leq  \exp(D)\log\left(\left\vert \Mmat_{t+1}\right\vert/\left\vert \Mmat_{1}\right\vert\right) & (\text{Lemma 4.6 of \cite{hazan2016intro}})\\
		&\leq  \exp(D)d\log(1+t/(4d))\; ,
\end{align*}
where we last used \cref{lemma:determinant_trace_inequality} along with $\elltwo{x_s}^2\leq \dot\mu(a_s\transp\theta'_{s+1})\leq 1/4$. 
\end{proof}

\subsection{Proof of \cref{prop:confset}}\label{app:proof_our_confset}
We prove below \cref{prop:confset} from the main paper. It justifies the confidence sets used in \texttt{OFU-ECOLog}. 

In this context, we have $\Theta_t \equiv \Theta$, the set returned by the warm-up procedure run with the conditions of \cref{prop:diameter_warmup} and $\eps_s = 1/s$. 
\detailtrue
\thmourconfset*
\detailfalse

The function $\sigma_t(\delta)$ is defined in \cref{eq:def_sigma} and checks $\sigma_t(\delta) \leq \const S^2d\log(t/\delta)$.

\begin{proof}
By \cref{prop:diameter_warmup} we know that $\diam_\mcal{A}(\Theta)\leq 1$ so we can set $D=1$. For the rest of the proof we assume that the event $\{\theta_\star\in\Theta\}$ holds - this happens with probability at least $1-\delta$ according to \cref{prop:diameter_warmup}. \cref{thm:conf_set} therefore applies since $\Theta$ is convex and compact. This yields:
\begin{align*}
	\mbb{P}\left( \forall t\geq 1, \; \left\|\theta_\star-\theta_{t+1}\right\|^2_{\Wmat_t}\leq 
	8S^2 + 4\sum_{s=1}^{t} s\varepsilon_s^2 + 2 + 9\nu_t(\delta) + 18\exp(1)d\log(1+t/(4d)) \right)\geq 1-\delta\; .
\end{align*}
After a classic bound on the harmonic function; for $t\geq 1$:
\begin{align*}
\sum_{s=1}^{t} s\varepsilon_s^2 &= \sum_{s=1}^{t} 1/s\leq 1 + \log(t)\; ,
\end{align*}
we are left to apply a naive union bound with the event $\{\theta_\star\in\Theta\}$ to finish the proof. 
\end{proof}

\subsection{A Data-Dependent Version}

The following result justifies the confidence regions used in \texttt{ada-OFU-ECOLog}.

\begin{restatable}{prop}{thm:conf_set_data_dependent}\label{thm:conf_set_dd}
Let $\delta\in(0,1]$ and $\{(\theta_t, \Wmat_t, \Theta_t)\}_t$ maintained by \cref{alg:dilutedofuecolog}. Then:
\begin{align*}
		\mbb{P}\left(\forall t\geq 1\;, \theta_\star\in\Theta_t \text{ and } \elltwo{\theta_\star-\theta'_{t+1}}_{\Wmat_{t+1}}^2 \leq \eta_t(\delta) \right)\geq 1-2\delta\; .
	\end{align*}
\end{restatable}

The function $\eta_t(\delta)$ is defined in \cref{eq:def_eta} and checks $\eta_t(\delta)\leq \const S^2d\log(t/\delta)$.

\begin{proof}
This result can be easily be retrieved from the proof of \cref{thm:conf_set}. The sets:
\begin{align*}
	\Theta_{t+1} = \Big\{ \theta, \big\| \theta-\hat\theta^\mcal{H}_{t+1}\big\|^2_{\Vmat^\mcal{H}_t}\leq \beta_t(\delta) \Big\}\; ,
\end{align*}
maintained in \cref{alg:dilutedofuecolog} are indeed compact and convex. Further, they contain $\theta_\star$ with high probability:
\begin{align*}
	1-\delta &\leq \mbb{P}\Big(\forall t\geq 1, \; \big\| \theta_\star -\hat\theta^\mcal{H}_{t+1}\big\|^2_{\Hmat^\mcal{H}_t(\theta_\star)}\leq \beta_t(\delta)\Big) &(\text{\cref{prop:concentraitonnoproj}})\\
	&\leq \mbb{P}\Big(\forall t\geq 1, \; \big\| \theta_\star -\hat\theta^\mcal{H}_{t+1}\big\|^2_{\mbold{V}^\mcal{H}_t}\leq \beta_t(\delta)\Big) &(\text{\cref{eq:Ht2Vt}})\\
	&=  \mbb{P}\left(\forall t\geq 1,\theta_\star \in\Theta_{t}\right)\; .
\end{align*}
Further, recall that the inequality:
\begin{align*}
	 \dot\mu(a_s\transp\bar\theta_s) \leq 2\dot\mu(a_s\transp\theta_{s+1})\; ,
\end{align*}
holds \emph{by construction} in \cref{alg:dilutedofuecolog}. When it is not satisfied, the couple $(a_s,r_{s+1})$ is not fed to the \texttt{ECOLog} procedure. This essentially allows to replace $\exp(D)$ in \cref{lemma:last_term_end} by a constant factor (here, 2). From there, following the demonstration of \cref{thm:conf_set} up to straight-forward adaptations (\emph{e.g} to deal with the fact that some rounds are ignored from the learning when the above inequality is not satisfied) yields that under the event $\{\forall t\geq 1, \theta_\star\in\Theta_{t}\}$:
\begin{align*}
	\mbb{P}\Big(\forall t\geq 1, \; \big\| \theta_\star -\theta'_{t+1}\big\|^2_{\Wmat_{t+1}}\leq \eta_t(\delta) \Big) \geq 1-\delta\; . 
\end{align*}

 A union bound finishes the proof.
\end{proof}

%% file: appendix/regret.tex
\newpage
\section{REGRET BOUNDS}\label{app:regret}

To reduce clutter and fit with the notations adopted in the main text, we go back in this section to identifying $\theta_t$ and its $\eps$-approximation $\theta'_{t}$. This does not impact the validity of the regret bounds - the effects of optimization errors are fully dealt with in the radius of the confidence sets we designed in \cref{app:our_conf_set}.

\subsection{Proof of \cref{thm:regret_ofuecolog}}\label{app:regret_ofu}
\detailtrue
\thmregretoful*
\detailfalse

\input{algos/ofu_ecolog_app_algo}

\begin{proof}
According to \cref{prop:confset} (its detailed version in \cref{app:proof_our_confset}) setting $\tau = \const\kappa S^6d^2\log(T/\delta)$ ensures:
\begin{align*}
	\mbb{P}\left(\theta_\star\in\Theta \text{ and } \;  \left\|\theta_{t}-\theta_\star\right\|^2_{\Wmat_t}\leq \sigma_t(\delta) \text{ for all } t\geq \tau+1\right)&\geq 1-2\delta\; ,
\end{align*}
In the rest of the proof \underline{we assume that the above event, denoted $E_\delta$, holds}.

Since $\mu(\cdot)\in(0,1)$ the regret incurred during warm-up can be directly bounded by $\tau$. Therefore for $T\geq \tau +1$;
\begin{align*}
	\regret(T) &\leq \const \kappa S^6d^2\log(T/\delta) +  \sum_{t=\tau+1}^T \mu(a_\star\transp\theta_\star)-\mu(a_t\transp\theta_\star) \\
	&\leq \const \kappa S^6d^2\log(T/\delta) + R(T)\; ,
\end{align*}
where we defined $R(T) =  \sum_{t=\tau+1}^T \mu(a_\star\transp\theta_\star)-\mu(a_t\transp\theta_\star)$. To control this term we follow the usual strategy for bounding the regret of optimistic algorithms, and re-use tools introduced by \cite{faury2020improved,abeille2020instance} - adapted to our confidence set. In the following, we denote for $t\geq \tau +1$:
\begin{align*}
	(a_t,\tilde\theta_t) \in \argmax_{\mcal{A},\mcal{C}_{t}(\delta)} a\transp\theta\; .
\end{align*}
where $\mcal{C}_t(\delta) = \{ \theta, \;   \left\|\theta_{t}-\theta\right\|^2_{\Wmat_t}\leq \sigma_t(\delta)\}$. 
Because $E_\delta$ holds this implies that the couple $(a_t,\tilde\theta_t)$ is optimistic. Formally: $a_t\transp\tilde\theta_t \geq a_\star\transp\theta_\star$. 
We start by tying the regret to the prediction error of $\tilde\theta_{t+1}$ and continue with a second-order Taylor expansion. 
\begin{align*}
	R(T) &= \sum_{t=\tau+1}^T \mu(a_\star\transp\theta_\star)-\mu(a_t\transp\theta_\star)\\
	&\leq \sum_{t=\tau+1}^T \mu(a_t\transp\tilde\theta_t)-\mu(a_t\transp\theta_\star) &(\text{optimism, } \mu\nearrow)\\
	&\leq \sum_{t=\tau+1}^T \dot\mu(a_t\transp\theta_\star) a_t\transp(\tilde\theta_t-\theta_\star) + \tilde\alpha(a_t\transp\theta_\star,a_t\transp\tilde\theta_t)(a_t\transp(\tilde\theta_t-\theta_\star))^2 &(\text{Taylor, }\vert\ddot\mu\vert\leq \dot\mu)\\
	&=: R_1(T) + R_2(T)\; .
\end{align*} 
	Above, we defined $R_1(T)=\sum_{t=\tau+1}^T \dot\mu(a_t\transp\theta_\star) a_t\transp(\tilde\theta_t-\theta_\star)$ and $R_2(T)=\sum_{t=\tau+1}^T \tilde\alpha(a_t\transp\theta_\star,a_t\transp\tilde\theta_t)(a_t\transp(\tilde\theta_t-\theta_\star))^2$. 
We start by bounding $R_2(T)$;
\begin{align*}
	R_2(T)  &\leq  \sum_{t=\tau+1}^T  (a_t\transp(\tilde\theta_t-\theta_\star))^2/2 &(\vert \dot\mu\vert \leq 1)\\ 
	&\leq  \sum_{t=\tau+1}^T \|a_t\|_{\Wmat_t^{-1}}^2\|\tilde\theta_t-\theta_\star\|_{\Wmat_t}^2/2 &(\text{Cauchy-Schwarz})\\
	&\leq  2\sigma_t(\delta)\sum_{t=\tau+1}^T \|a_t\|_{\Wmat_t^{-1}}^2 &(\tilde\theta_t,\theta_\star\in\mcal{C}_t(\delta))\\
	&\leq \const d S^2 \log(T/\delta) \sum_{t=\tau+1}^T \|a_t\|_{\Wmat_t^{-1}}^2 &(\text{\cref{eq:def_sigma}})\\
	&\leq   \const d S^2 \log(T/\delta) \sum_{t=\tau+1}^T \|a_t\|_{\Vmat_t^{-1}}^2 &\\
	&\leq   \const d^2 \kappa S^2 \log(T/\delta)^2 &(\text{\cref{lemma:ellipticalpotential}})\\
\end{align*}
We last applied \cref{lemma:ellipticalpotential} with $x_t = a_t/\sqrt{\kappa}$, and proceeded with some simple upper-bounding operations. The second to last inequality is a consequence of  \citet[Lemma 9]{abeille2020instance} which ensures:
\begin{align*}
	\dot\mu(a_s\transp\theta'_{s+1}) &\geq \dot\mu(a_s\transp\theta_\star) \exp(-\vert a_s\transp(\theta'_{s+1}-\theta_\star)\vert) \\
	&\geq \dot\mu(a_s\transp\theta_\star) \exp(-1)  &(\theta_\star,\theta'_{s+1}\in\Theta, \, \diam_\mcal{A}(\Theta)\leq 1) \\
	&\geq \exp(-1)\kappa\; .
\end{align*}

We now turn our attention to $R_1(T)$. 
\begin{align*}
	R_1(T) &=  \sum_{t=\tau+1}^T \dot\mu(a_t\transp\theta_\star) a_t\transp(\tilde\theta_t-\theta_\star) \\
	&\leq  \sum_{t=\tau+1}^T  \sqrt{ \dot\mu(a_t\transp\theta_\star)}\sqrt{ \exp(\vert a_t\transp(\theta_{t+1}-\theta_\star)\vert)  \dot\mu(a_t\transp\theta_{t+1})} a_t\transp(\tilde\theta_t-\theta_\star) &(\text{Lemma 9 of \cite{abeille2020instance}}))\\
	&\leq \sqrt{e}  \sum_{t=\tau+1}^T  \sqrt{ \dot\mu(a_t\transp\theta_\star)} \sqrt{\dot\mu(a_t\transp\theta_{t+1})} a_t\transp(\tilde\theta_t-\theta_\star) &(\diam_\mcal{A}(\Theta)\leq 1)\\
	&\leq \sqrt{e}  \sum_{t=\tau+1}^T   \sqrt{ \dot\mu(a_t\transp\theta_\star)}\sqrt{\dot\mu(a_t\transp\theta_{t+1})} \|a_t\|_{\Wmat_{t}^{-1}}\|\tilde\theta_t-\theta_\star\|_{\Wmat_t} &(\text{Cauchy-Schwarz})\\
	&\leq e  \sum_{t=\tau+1}^T     \sqrt{ \dot\mu(a_t\transp\theta_\star)}\sqrt{ \dot\mu(a_t\transp\theta_{t+1}) }\|a_t\|_{\Wmat_{t}^{-1}}\left(\|\theta_t-\theta_\star\|_{\Wmat_t}  + \|\tilde\theta_t-\theta_t\|_{\Wmat_t} \right)&(\text{Triangle ineq.})\\
	&\leq 2e\sqrt{\sigma_T(\delta)} \sum_{t=\tau+1}^T     \sqrt{ \dot\mu(a_t\transp\theta_\star)}\sqrt{ \dot\mu(a_t\transp\theta_{t+1})} \|a_t\|_{\Wmat_{t}^{-1}}&(\tilde\theta_t,\theta_\star\in\mcal{C}_t(\delta))\\
	&\leq \const S\sqrt{d\log(T/\delta)}  \sqrt{\sum_{t=\tau+1}^T \dot\mu(a_t\transp\theta_\star)}  \sqrt{\sum_{t=\tau+1}^T \dot\mu(a_t\transp\theta_{t+1}) \|a_t\|^2_{\Wmat_{t}^{-1}}}&(\text{Cauchy-Schwarz})\\
	&\leq \const S\sqrt{d\log(T/\delta)}\sqrt{d\log(1+T/d)}\sqrt{\sum_{t=\tau+1}^T \dot\mu(a_t\transp\theta_{\star})} &(\text{\cref{lemma:ellipticalpotential}})\\
\end{align*}
where \cref{lemma:ellipticalpotential} was used with $x_t = \sqrt{\dot\mu(a_t\transp\theta_{t+1})}a_t$ (and $\dot\mu\leq 1$).
From then, we can directly follow the proof of Theorem 1 from \cite{abeille2020instance} (more precisely, follow the reasoning employed in their Section C.1 page 18) for which we extract that:
\begin{align*}
	\sum_{t=\tau+1}^T \dot\mu(a_t\transp\theta_{\star}) \leq R_T + T\dot\mu(a_\star\transp\theta_\star)\; .
\end{align*}
Assembling the bounds on $R_2(T)$, $R_1(T)$ and $\sum_{t=\tau+1}^T \dot\mu(a_t\transp\theta_{\star})$ we obtain that:
\begin{align*}
		R(T) \leq \const \kappa S^2d^2\log(T/\delta)^2 + \const Sd\log(T/\delta)\sqrt{R_T + T\dot\mu(a_\star\theta_\star)}\; .
\end{align*}
Because $x^2-bx-c\leq 0\, \Rightarrow x^2\leq 2b^2+2c$  we have:
\begin{align*}
	R(T) \leq \const Sd\log(T/\delta)\sqrt{T\dot\mu(a_\star\theta_\star)} + \const \kappa S^2d^2\log(T/\delta)^2\; ,
\end{align*}
which concludes the proof.  
\end{proof}

\begin{rem}
	The scaling w.r.t $S$ of the regret's second-order term is driven by the length $\tau$ of the warm-up phase. As anticipated in \cref{rem:about_proj} this scaling is reduced when $\|\hat\theta_\tau\|\leq S$ which often happens in practice. In this case, we obtain a second-order term which exactly matches the one of \cite{abeille2020instance}.
\end{rem}

\subsection{The \texttt{TS-ECOLog} algorithm}\label{app:regret_ts}
In this section we introduce the TS version of \texttt{OFU-ECOLog} whose pseudo-code is provided in \cref{alg:tsecolog}.
\input{algos/ts_pseudo_algo}

The algorithm display little novelty compared to the linear case studied by \cite{agrawal2013thompson,abeille2017linear}. The only difference is a rejection sampling step on $\Theta$. The analysis is also similar, up to minor modifications. The following statement provides a regret guarantee for \texttt{TS-ECOLog}.

\begin{thm} Let $\delta\in(0,1]$ and $\mcal{D}^{\text{\textnormal{TS}}}$ a distribution satisfying Definition 1 of \citet{abeille2017linear}. Setting $\tau= \const \kappa S^6 d^2 \log(T/\delta)^2$ ensures that the regret of \textnormal{\texttt{TS-ECOLog}($\delta,\tau,\mcal{D}^{\text{TS}}$)} satisfies with probability at least $1-\delta$:
\begin{align*}
	\regret(T) \leq \const Sd^{3/2}\log(T/\delta)\sqrt{T\dot\mu(a_\star\transp\theta_\star)} + \const S^6\kappa d^3 \log(T/\delta)^2\; .
\end{align*}
\end{thm}

\begin{proof}
According to \cref{prop:confset} (its detailed version in \cref{app:proof_our_confset}) setting $\tau = \const\kappa S^6d^2\log(T/\delta)$ ensures:
\begin{align*}
	\mbb{P}\left(\theta_\star\in\Theta \text{ and }  \left\|\theta_{t}-\theta_\star\right\|^2_{\Wmat_t}\leq \sigma_t(\delta) \text{ for all } t\geq \tau+1\right)&\geq 1-2\delta\; .
\end{align*}

As in the proof of \cref{thm:regret_ofuecolog} \underline{we assume that the above event, denoted $E_\delta$, holds}.

We decompose the regret as:
\begin{align*}
	\regret(T) &\leq \tau + \sum_{t=\tau+1}^T \mu(a_\star\transp\theta_\star) - \mu(a_t\transp\theta_\star)\\
	&= \tau + \sum_{t=\tau+1}^T \mu(a_\star\transp\theta_\star) - \mu(a_t\transp\tilde\theta_t) + \sum_{t=\tau+1}^T \mu(a_t\transp\tilde\theta_t) - \mu(a_t\transp\theta_\star) \\
	&\leq \const \kappa S^6 d^2 \log(T/\delta)^2 + R^{\text{TS}}(T) + R^{\text{PRED}}(T)  \; .
\end{align*}
Above, we defined $R^{\text{TS}}(T) =  \sum_{t=\tau+1}^T\mu(a_\star\transp\theta_\star) - \mu(a_t\transp\tilde\theta_t)$ and $R^{\text{PRED}}(T) = \sum_{t=\tau+1}^T \mu(a_t\transp\tilde\theta_t) - \mu(a_t\transp\theta_\star)$. To bound $R^{\text{PRED}}(T)$ one can directly follow the strategy employed in \cref{app:regret_ofu}. The only difference comes from the radius of the ``effective'' confidence set that is used - inflated by $\sqrt{d}$ because of the concentration properties of $\mcal{D}^{\text{TS}}$. This leads to:
\begin{align*}
	R^{\text{PRED}}(T) &\leq  \const Sd^{3/2}\log(T/\delta)\sqrt{T\dot\mu(a_\star\theta_\star)} + \const \kappa S^2d^3\log(T/\delta)^2
\end{align*}
We now turn to $R^{\text{TS}}(T)$. Following \cite{abeille2017linear} we denote $J(\theta) = \max_{a\in\mcal{A}} a\transp\theta$. We have:
\begin{align}
	R^{\text{TS}}(T) &= \sum_{t=\tau+1}^T\mu(a_\star\transp\theta_\star) - \mu(a_t\transp\tilde\theta_t)& \notag\\
	&= \sum_{t=\tau+1}^T \alpha(a_\star\transp\theta_\star, a_t\transp\tilde\theta_t) (a_\star\transp\theta_\star - a_t\transp\tilde\theta_t) &(\text{exact first-order Taylor})\notag\\
	&= \sum_{t=\tau+1}^T \alpha(J(\theta_\star), J(\tilde\theta_t)) (J(\theta_\star) - J(\tilde\theta_t)) &(\text{def. of } J)\label{eq:tmp_rts}
\end{align}
By convexity of $J$ along with the computations of its sub-gradients (see Section C of \cite{abeille2017linear});
\begin{align*}
	\vert J(\theta_\star)-J(\tilde\theta_t)\vert   &\leq \max\left\{ \vert \nabla J(\theta_\star)\transp(\theta_\star - \tilde\theta_t)\vert,  \vert \nabla J(\tilde\theta_t)\transp(\theta_\star - \tilde\theta_t)\vert\right\}  &(\text{convexity of } J)\\
	&\leq \max\left\{ \vert a_\star \transp(\theta_\star - \tilde\theta_t)\vert, \vert a_t \transp(\theta_\star - \tilde\theta_t)\vert \right\}  & (\nabla J(\theta) = \argmax_{a\in\mcal{A}} a\transp\theta)\\
	&\leq \diam_\mcal{A}(\Theta) &(\tilde\theta_t, \theta_\star\in\Theta)\\
	&\leq 1\;. &(\text{\cref{prop:diameter_warmup}})
\end{align*}
Therefore:
\begin{align*}
	\alpha(J(\theta_\star),J(\tilde\theta_t)) &= \int_{v=0}^1 \dot\mu(J(\theta_\star) + v(J(\tilde\theta_t) - J(\theta_\star))dv\\
	&\leq \dot\mu(J(\theta_\star))   \int_{v=0}^1\exp(v\vert J(\tilde\theta_t) - J(\theta_\star)\vert )dv &(\text{Lemma 9 of \cite{abeille2020instance}})\\
	&\leq \dot\mu(J(\theta_\star))   \int_{v=0}^1\exp(v)dv &(\vert J(\theta_\star)-J(\tilde\theta_t)\leq 1) \\
	&\leq 2\dot\mu(J(\theta_\star))\; .
\end{align*}
Plugging the above inequality in \cref{eq:tmp_rts} yields:
\begin{align*}
	R^{\text{TS}}(T) &\leq 2\dot\mu(J(\theta_\star)) \sum_{t=\tau+1}^T J(\theta_\star) - J(\tilde\theta_t)\\
	&= 2\dot\mu(a_\star\transp\theta_\star)  \sum_{t=\tau+1}^T J(\theta_\star) - J(\tilde\theta_t)\; .
\end{align*}
From then on we can follow the proof of \cite{abeille2017linear} which in the linear case studies exactly $\sum_t J(\theta_\star) - J(\tilde\theta_t)$. Directly following their line of proof yields $\sum_t J(\theta_\star) - J(\tilde\theta_t) \lesssim \const \sqrt{d}\sqrt{\sigma_T(\delta)} \sum_t \| a_t\|_{\Wmat_t^{-1}} + \sqrt{T}$. This concludes the proof since $\sigma_T(\delta) \leq \const S^2 d\log(T/\delta)$ and:
\begin{align*}
	\dot\mu(a_\star\transp\theta_\star) \sum_{t=\tau+1}^T  \| a_t\|_{\Wmat_t^{-1}}  &= \sqrt{\dot\mu(a_\star\transp\theta_\star)}\sum_{t=\tau+1}^T \sqrt{\dot\mu(a_\star\transp\theta_\star)} \| a_t\|_{\Wmat_t^{-1}}\\
	&\leq \const \sqrt{\dot\mu(a_\star\transp\theta_\star)}\sum_{t=\tau+1}^T \sqrt{\dot\mu(a_\star\transp\theta_{t+1})} \sqrt{\exp(\vert a_\star\transp(\theta_\star-\theta_{t+1})\vert)}\| a_t\|_{\Wmat_t^{-1}} \\
	&\leq \const \sqrt{\dot\mu(a_\star\transp\theta_\star)}\sum_{t=\tau+1}^T \sqrt{\dot\mu(a_\star\transp\theta_{t+1})} \sqrt{\exp(2\diam_\mcal{A}(\Theta))}\| a_t\|_{\Wmat_t^{-1}}  &(\theta_{t+1},\theta_\star\in\Theta)\\
	&\leq \const \sqrt{\dot\mu(a_\star\transp\theta_\star)}\sum_{t=\tau+1}^T \sqrt{\dot\mu(a_\star\transp\theta_{t+1})}\| a_t\|_{\Wmat_t^{-1}}  &(\diam_\mcal{A}(\Theta)\leq 1)\\
	&\leq  \const \sqrt{T\dot\mu(a_\star\transp\theta_\star)} \sqrt{d\log(T)} \; .&(\text{Cauchy-Schwarz, \cref{lemma:ellipticalpotential}})
\end{align*}
\end{proof}

\subsection{Proof of \cref{thm:dilutedofuecolog}} %
\label{app:ada_ofu_regret}
\detailtrue
\regretdiluted*
\detailfalse

\input{algos/diluted_pseudo_algo_app}

\begin{proof}
	We denote $\Tau$ the set of rounds at which condition \eqref{eq:cond_measurable} breaks. Formally;
	\begin{align*}
		 \Tau \defeq \left\{t\in[T], \;  \dot\mu(a_t\transp\bar\theta_t)\geq 2 \dot\mu(a_t\transp\theta_t^1) \text{ or }  \dot\mu(a_t\transp\bar\theta_t)\geq 2 \dot\mu(a_t\transp\theta_t^0)\right\} \; .
	\end{align*}
	We claim the following result bounding the cardinality of $\Tau$. The proof is deferred to \cref{app:boundingtau}. 
	
	\begin{restatable}{lemma}{boundingTau}\label{lemma:bounding_tau}
	The following inequality holds:
	\begin{align*}
		\left\vert \Tau \right\vert \leq \const S^6 \kappa d^2 \log(T/\delta)^2 \; .
	\end{align*}
	\end{restatable}
	We follow a naive (but sufficient) bounding strategy. For rounds $t\in\Tau$ we crudely bound the instantaneous regret by its maximal value (\emph{e.g.} 1);
	\begin{align*}
		\regret(T) 
		&\leq  \left\vert \Tau\right\vert + \sum_{ t\in[T]\setminus \Tau} \dot\mu(a_{\star,t}\transp\theta_\star)-\dot\mu(a_t\transp\theta_\star) &(\mu\in(0,1))\\
		&\leq \const S^6 \kappa d^2 \log(T/\delta)^2 + R_T &(\text{\cref{lemma:bounding_tau}})
	\end{align*}
	where $R_T \defeq  \sum_{t\notin\Tau} \dot\mu(a_{\star,t}\transp\theta_\star)-\dot\mu(a_t\transp\theta_\star)$. In the following, we denote for $t\notin\tau$:
\begin{align*}
	(a_t,\tilde\theta_t) \in \argmax_{\mcal{A}_t,\mcal{C}_{t}(\delta)} a\transp\theta\; .
\end{align*}
where $\mcal{C}_t(\delta) = \{ \theta, \;   \left\|\theta_{t}-\theta\right\|^2_{\Wmat_t}\leq \eta_t(\delta)\}$ and $\eta_t(\delta)$ is defined in \cref{eq:def_eta}. 
In the following, \underline{we assume that the following event holds}:
\begin{align*}
	E_\delta = \left\{ t\in[T]\setminus \Tau, \; \theta_\star \in\mcal{C}_t(\delta) \cap\Theta_t\right\} \; ,
\end{align*} 
This happens \underline{with probability at least $1-2\delta$ }according to \cref{thm:conf_set_dd}. This implies that the couple $(a_t,\tilde\theta_t)$ is optimistic. Formally: $a_t\transp\tilde\theta_t \geq a_{\star,t}\transp\theta_\star$. 
 Therefore: 
\begin{align*}
	R(T) &= \sum_{ t\in[T]\setminus \Tau} \mu(a_{\star,t}\transp\theta_\star)-\mu(a_t\transp\theta_\star)\\
	&\leq \sum_{ t\in[T]\setminus \Tau} \mu(a_t\transp\tilde\theta_t)-\mu(a_t\transp\theta_\star) &(\text{optimism, } \mu\nearrow)\\
	&\leq \sum_{ t\in[T]\setminus \Tau} \dot\mu(a_t\transp\theta_\star) a_t\transp(\tilde\theta_t-\theta_\star) + \tilde\alpha(a_t\transp\theta_\star,a_t\transp\tilde\theta_t)(a_t\transp(\tilde\theta_t-\theta_\star))^2 &(\text{Taylor, }\vert\ddot\mu\vert\leq \dot\mu)\\
	&=: R_1(T) + R_2(T)\; .
\end{align*} 
The bound on $R_2(T)$ is directly extracted from the proof of \cref{thm:regret_ofuecolog} presented in \cref{app:regret_ofu}.
\begin{align*}
	R_2(T)  \leq  \const d^2 \kappa S^2 \log(T/\delta)^2\; .
\end{align*}

The story is slightly different for $R_1(T)$ and the proof laid out in \cref{app:regret_ofu} needs to be slightly adapted. We need to differentiate the rounds where $\dot\mu(a_t\transp\theta_\star)\leq \dot\mu(a_t\transp\theta_{t+1})$ and the rounds where $\dot\mu(a_t\transp\theta_\star)\geq \dot\mu(a_t\transp\theta_{t+1})$. \underline{In what follows we focus only on the latter} (for the former we can directly adapt the approach laid out in \cref{app:regret_ofu}).

\begin{align*}
	R_1(T) &=  \sum_{ t\in[T]\setminus \Tau} \dot\mu(a_t\transp\theta_\star) a_t\transp(\tilde\theta_t-\theta_\star) \\
	&\leq  \sum_{ t\in[T]\setminus \Tau} \dot\mu(a_t\transp\theta_{t+1}) a_t\transp(\tilde\theta_t-\theta_\star) + a_t\transp(\tilde\theta_t-\theta_\star)\vert a_t\transp(\theta_\star-\theta_{t+1})\vert &(\text{Taylor, } \vert\ddot\mu\vert\leq 1)\\
	&\leq  \sum_{t\in[T]\setminus \Tau} \dot\mu(a_t\transp\theta_{t+1}) \|a_t\|_{\Wmat_{t}^{-1}}\|\tilde\theta_t-\theta_\star\|_{\Wmat_t} +  \|a_t\|_{\Wmat_{t}^{-1}}^2\|\tilde\theta_t-\theta_\star\|_{\Wmat_t}\|\theta_{t+1}-\theta_\star\|_{\Wmat_t}&(\text{Cauchy-Schwarz}) \\
	&\leq \const \sqrt{\eta_T(\delta)} \sum_{t\in[T]\setminus \Tau} \dot\mu(a_t\transp\theta_{t+1}) \|a_t\|_{\Wmat_{t}^{-1}} + \const \sqrt{\eta_T(\delta)} \sum_{t\in[T]\setminus \Tau} \|a_t\|^2_{\Wmat_{t}^{-1}}\|\theta_{t+1}-\theta_\star\|_{\Wmat_t} &(\tilde\theta_t\in\mcal{C}_t(\delta), E_\delta \text{ holds})\\
	&\leq \const \sqrt{\eta_T(\delta)} \sum_{t\in[T]\setminus \Tau} \dot\mu(a_t\transp\theta_{t+1}) \|a_t\|_{\Wmat_{t}^{-1}} + \const \sqrt{\eta_T(\delta)} \sum_{t\in[T]\setminus \Tau} \|a_t\|^2_{\Wmat_{t}^{-1}}\|\theta_{t+1}-\theta_\star\|_{\Wmat_{t+1}} &(\Wmat_t \succeq \Wmat_{t+1})\\
	&\leq \const \sqrt{\eta_T(\delta)} \sum_{t\in[T]\setminus \Tau} \dot\mu(a_t\transp\theta_{t+1}) \|a_t\|_{\Wmat_{t}^{-1}} + \const \eta_T(\delta)\sum_{t\in[T]\setminus \Tau} \|a_t\|^2_{\Wmat_{t}^{-1}} &(\theta_{t+1}\in\mcal{C}_{t+1}(\delta), E_\delta \text{ holds})
\end{align*}
The second term in the above inequality is bounded exactly as in $R_2(T)$; this yields:
\begin{align*}
	R_1(T) &\leq \const \sqrt{\eta_T(\delta)} \sum_{t\in[T]\setminus \Tau} \dot\mu(a_t\transp\theta_{t+1}) \|a_t\|_{\Wmat_{t}^{-1}} + \const \kappa S^2 d^2\log(T)^2&(\text{cf. bound on } R_2(T))\\
	&\leq \const \sqrt{\eta_T(\delta)} \sqrt{\sum_{t\in[T]\setminus \Tau} \dot\mu(a_t\transp\theta_{t+1})}\sqrt{ \sum_{t\in[T]\setminus \Tau} \dot\mu(a_t\transp\theta_{t+1}\|a_t\|^2_{\Wmat_{t}^{-1}}} + \const \kappa S^2 d^2\log(T/\delta)^2&(\text{Cauchy-Schwarz})\\
	&\leq \const Sd\log(T/\delta) \sqrt{\sum_{t\in[T]\setminus \Tau} \dot\mu(a_t\transp\theta_{t+1})} + \const \kappa S^2 d^2\log(T/\delta)^2&(\text{\cref{lemma:ellipticalpotential}})\\
	&\leq \const Sd\log(T/\delta) \sqrt{\sum_{t\in[T]\setminus \Tau} \dot\mu(a_t\transp\theta_\star)} + \const \kappa S^2 d^2\log(T/\delta)^2&(\text{by hyp.})
\end{align*}

Again, by following the proof of Theorem 1 from \cite{abeille2020instance} we get that:
\begin{align*}
	\sum_{t\in[T]\setminus \Tau} \dot\mu(a_t\transp\theta_{\star}) \leq R_T + \sum_{t\in[T]\setminus \Tau} \dot\mu(a_{\star,t}\transp\theta_\star)\; .
\end{align*}
Assembling the different bounds and solving the implicit inequation on $R_T$ yields the announced result.
\end{proof}

	\subsubsection{Proof of \cref{lemma:bounding_tau}}\label{app:boundingtau}
	
	\boundingTau*
	
	\begin{proof}
	Denote for $u\in\{0,1\}$:
	\begin{align*}
		\Tau_u \defeq \left\{t\in[T], \;  \dot\mu(a_t\transp\bar\theta_t)\geq 2 \dot\mu(a_t\transp\theta_t^u)\right\}\; ,
	\end{align*}
	so that $\Tau = \Tau_0\cup\Tau_1$.
	By \citet[Lemma 9]{abeille2020instance} we know that:
	\begin{align*}
		\dot\mu(a_t\transp\bar\theta_t) \leq \dot\mu(a_t\transp\theta_t^u)\exp\left(\vert a_t\transp(\bar\theta_t-\theta_t^u)\vert\right)
	\end{align*}
	Therefore by straight-forward manipulations:
	\begin{align}
		t\in\Tau \Longrightarrow \exists u\in\{0,1\}\text{ s.t } \vert a_t\transp(\bar\theta_t-\theta_t^{u})\vert \geq \log(2)\; .\label{eq:csq_tau0}
	\end{align}
	We can now bound $\vert \Tau\vert$  thanks to the form of $\Theta_t$ (which contains $\bar\theta_t$ and $\theta_t^u$ by construction) and the Elliptical Potential lemma.
	\begin{align*}
		\vert \Tau \vert \log(2)^2 &\leq  \sum_{t\in\Tau} \left\vert a_t\transp(\bar\theta_t-\theta_t^u)\right \vert^2 \\
		&\leq \sum_{t\in\Tau_0} \left\|a_t\right\|^2_{(\Vmat_t^\mcal{H})^{-1}} \left\|\bar\theta_t-\theta_t^u\right\|^2_{\Vmat^{\mcal{H}}_t} &(\text{Cauchy-Schwarz})\\
		&\leq 4\beta_T(\delta) \sum_{t\in\Tau} \left\|a_t\right\|^2_{(\Vmat_t^\mcal{H})^{-1}} &(\bar\theta_t, \, \theta_t^u\in\Theta_t)\\
		&\leq 8 \kappa \beta_T(\delta) d\log(T) &(\text{\cref{lemma:ellipticalpotential}})\\
	\end{align*}
	We applied \cref{lemma:ellipticalpotential} with $x_s = a_s/\sqrt{\kappa}$, and after checking that in \cref{alg:dilutedofuecolog} the matrix $\Vmat^{\mcal{H}}_t$ is indeed updated in rounds $t\in\Tau$. This conclude the proof since $\beta_t(\delta)\leq \const S^6d\log(T/\delta)$.

	\end{proof}

%% file: algos/ofu_ecolog_app_algo.tex
\setcounter{algorithm}{0}
\begin{algorithm*}[h!]
   \caption{\texttt{OFU-ECOLog}}
   \label{alg:ofuecolog_app}
	\begin{algorithmic}
	\REQUIRE{failure level $\delta$, warm-up length $\tau$.}
	\STATE Set $\Theta \leftarrow \texttt{WarmUp}(\tau)$ (see \cref{alg:warmup}). \hfill\COMMENT{forced-exploration}
   \STATE Initialize $\theta_{\tau+1}\in\Theta$, $\Wmat_{\tau+1} \leftarrow \mbold{I}_d$ and $\mcal{C}_{\tau+1}(\delta)\leftarrow \Theta$.
   \FOR{$t\geq \tau+1$}
    \begin{spacing}{1.2}
   \STATE Play $a_t \in \argmax_{a\in\mcal{A}} \max_{\theta\in\mcal{C}_t(\delta)} a\transp\theta\; .$ \hfill\COMMENT{planning} 
   \STATE Observe reward $r_{t+1}$, construct loss $\ell_{t+1}(\theta)=\ell(a_t\transp\theta, r_{t+1})$.
   \STATE  Compute $(\theta_{t+1},\, \Wmat_{t+1})\leftarrow \texttt{ECOLog}(1/t, \Theta,\ell_{t+1},\Wmat_t,\theta_t)$ (see \cref{alg:ecolog}).\hfill\COMMENT{learning}
   \end{spacing}
   \vspace{2pt}
   \STATE Compute $\mcal{C}_{t+1}(\delta) \leftarrow\left \{\left\| \theta-\theta_{t+1} \right\|^2_{\Wmat_{t+1}} \leq \sigma_t(\delta)\right\}$.\hfill \COMMENT{$\sigma_t(\delta)$ is defined in \cref{eq:def_sigma}}
   \ENDFOR
\end{algorithmic}
\end{algorithm*}

%% file: algos/ts_pseudo_algo.tex
\setcounter{algorithm}{2}
\begin{algorithm}[h!]
   \caption{\texttt{TS-ECOLog}}
   \label{alg:tsecolog}
	\begin{algorithmic}
	\REQUIRE{failure level $\delta$, warm-up length $\tau$, distribution $\mcal{D}^{\text{TS}}$}
	\STATE Set $\Theta \leftarrow \texttt{WarmUp}(\tau)$ (see \cref{alg:warmup}). \hfill\COMMENT{forced-exploration}
   \STATE Initialize $\theta_{\tau+1}\in\Theta$, $\Wmat_{\tau+1} \leftarrow \mbold{I}_d$.
   \FOR{$t\geq \tau+1$}
    \begin{spacing}{1.2}
   \STATE  Set $\text{\texttt{reject}} \leftarrow \text{\texttt{true}}$  \hfill\COMMENT{sampling} 
   \WHILE{$\text{\texttt{reject}}$}
   	\STATE Sample $\eta \sim\mcal{D}^{\text{TS}}$, let $\tilde\theta_{t} = \theta_t + \sigma_t(\delta)\Wmat_{t}^{-1/2}\eta$.
	\STATE If $\tilde\theta_t\in\Theta$ set $\text{\texttt{reject}} \leftarrow \text{\texttt{false}}$
   \ENDWHILE
   \STATE Play $a_t \in\argmax_{a\in\mcal{A}} a\transp\tilde\theta_t$.
   \STATE Observe reward $r_{t+1}$, construct loss $\ell_{t+1}(\theta)=\ell(a_t\transp\theta, r_{t+1})$.
   \STATE  Compute $(\theta_{t+1},\, \Wmat_{t+1})\leftarrow \texttt{ECOLog}(1/t, \Theta,\ell_{t+1},\Wmat_t,\theta_t)$ (see \cref{alg:ecolog}).\hfill\COMMENT{learning}
   \end{spacing}
   \vspace{2pt}
   \ENDFOR
   \end{algorithmic}
\end{algorithm}

%% file: algos/diluted_pseudo_algo_app.tex
\setcounter{algorithm}{1}

\begin{algorithm*}[h!]
   \caption{\texttt{ada-OFU-ECOLog}}
   \label{alg:dilutedofuecolog_app}
	\begin{algorithmic}
	\REQUIRE{failure level $\delta$.}
   \STATE Initialize $\Theta_1 = \{ \| \theta \| \leq S \}$, $\mcal{C}_1(\delta) \leftarrow \Theta_1$, $\theta_{1}\in\Theta$, $\Wmat_{1} \leftarrow \mbold{I}_d$ and $\mcal{H}_1\leftarrow \emptyset$. 
   \FOR{$t\geq 1$}
    \begin{spacing}{1.2}
    \STATE Play $a_t \in \argmax_{a\in\mcal{A}} \max_{\theta\in\mcal{C}_t(\delta)} a\transp\theta$, observe reward $r_{t+1}$.
    \STATE Compute the estimators $\theta_t^0$, $\theta_t^1$ (see \cref{eq:defbaru}) and $\bar\theta_t$. 
    \IF{$\dot{\mu}(a_t\transp\bar\theta_t)\leq 2\dot{\mu}(a_t\transp\theta_t^0)$ and $\dot{\mu}(a_t\transp\bar\theta_t)\leq 2\dot{\mu}(a_t\transp\theta_t^1)$} 
    	\STATE Form the loss $\ell_{t+1}$ and compute $(\theta_{t+1},\, \Wmat_{t+1})\leftarrow\texttt{ECOLog}(1/t, \Theta_t,\ell_{t+1},\Wmat_t,\theta_t)$. %
	\STATE Compute $\mcal{C}_{t+1}(\delta) \leftarrow \left\{\| \theta-\theta_{t+1} \|^2_{\Wmat_{t+1}} \leq \eta_t(\delta)\right\}$, set $\mcal{H}_{t+1} \leftarrow \mcal{H}_t$.\hfill\COMMENT{$\eta_t(\delta$) is defined in \cref{eq:def_eta}}
	\ELSE
		\STATE Set $\mcal{H}_{t+1} \leftarrow \mcal{H}_t\cup \{ a_t, r_{t+1} \}$ and compute $\hat\theta^{\mcal{H}}_{t+1} = \argmin \sum_{(a,r)\in\mcal{H}_{t+1}} \ell(a\transp\theta, r) + \gamma_t(\delta)\|\theta\|^2$.
		\STATE  Update $\mbold{V}^{\mcal{H}}_{t}\leftarrow \sum_{a\in\mcal{H}_{t+1}} aa\transp/\kappa + \gamma_t(\delta)\mbold{I}_d$, $\theta_{t+1}\leftarrow \theta_t$ and $\Wmat_{t+1}\leftarrow \Wmat_t$.
		\STATE Compute $\Theta_{t+1} = \left\{ \| \theta - \hat\theta^{\mcal{H}}_{t+1} \|^2_{\Vmat^{\mcal{H}}_{t}} \leq \beta_t(\delta)\right\}\cap \Theta_1\; .$\hfill\COMMENT{$\beta_t(\delta)$ is defined in \cref{eq:def_beta}}
    \ENDIF
   \end{spacing}
   \vspace{0pt}
   \ENDFOR
\end{algorithmic}
\end{algorithm*}

%% file: appendix/cost.tex
\newpage
\section{Computational Costs}\label{app:ccost}

\subsection{Proof of \cref{prop:per_round_cost,prop:dilutedcomplexity}}
\label{app:cost}
The goal of this section is to examine the per-round complexity of the algorithms laid out in the main paper. 

\subsubsection{Per-Round Cost of \texttt{ECOLog}}

We start by the main computational bottleneck of our approach, which is the \texttt{ECOLog} procedure (see its sequential form in \cref{alg:ecolog2}). It involves computing $\theta'_{t+1}$ - an $\eps_t$-approximation (in $\ell_2$-norm) of $\theta_{t+1}$, and
updating the matrix $\Wmat_{t+1}$ (along with its inverse which will be used for the planning mechanism). 
We claim the following result, a slightly more detailed version of \cref{prop:mt_cost_ecolog} in the main text. 

\begin{restatable}{prop}{propcomputationalcost}\label{prop:computational_cost}
Fix $t\in\mbb{N}^+$. Assume that $\Theta_t$ is a bounded and closed ellipsoid and $\eps_t>0$. Completing round $t$ of \textnormal{\texttt{ECOLog}} can be done within $\bigo{d^2\log(\diam(\Theta_t))/\varepsilon_t)^2}$ operations.
\end{restatable}

\begin{proof}
Given $\theta'_{t+1}$ the matrix $\Wmat_{t+1}$ can be updated at cost $\bigo{d^2}$ since:
\begin{align*}
	\Wmat_{t+1} = \Wmat_t + \dot\mu(a_t\transp\theta'_{t+1})a_ta_t \; . 
\end{align*}
The cost of maintaining $\Wmat_{t+1}^{-1}$ is the same thanks to the Sherman-Morrison formula. 
The main computational complexity therefore stems from the computation of $\theta'_{t+1}$. Recall:
\begin{align*}
	\theta_{t+1} = \argmin_{\theta\in\Theta_t} \frac{1}{2+D}\left\|\theta-\theta'_t\right\|_{\Wmat_{t}}^2 + \ell_{t+1}(\theta)\; ,
\end{align*}
where $D\geq \diam_\mcal{A}(\Theta_t)$.  Let $\Wmat_t =\Lmat_t\Lmat_t\transp$ the Cholesky decomposition of $\Wmat_t$ (it exists since $\Wmat_t$ is p.s.d). By denoting $z_t = \Lmat_t\transp\theta'_t$, performing the change of variable $z\leftarrow\Lmat_t\transp\theta$ and removing constants we obtain:
\begin{align*}
	\theta_{t+1} =  \Lmat_t\minustransp \argmin_{\Lmat_t\minustransp z\in\Theta} \left( \bar{L}_{t+1}(z) \defeq \frac{1}{2+D}\|z\|^2 + \frac{2}{2+D}z\transp z_t + \ell_{t+1}(\Lmat_t\minustransp z)\right)\; . 
\end{align*}
By direct computations:
\begin{align*}
	\nabla^2 \bar{L}_{t+1}(z) = (1+D/2)^{-1}\mbold{I}_d + \dot{\mu}(a_t\transp\Lmat_t\minustransp z)\Lmat_t^{-1}a_ta_t\transp\Lmat_t\minustransp\; .
\end{align*}
proving that for all $z\in\mbb{R}^d$ (using the fact that $\dot\mu\in[0,1/4]$ and $\Wmat_{t}\succeq \mbold{I}_d$) :
\begin{align*}
	0 \prec (1+D/2)^{-1} \preceq \nabla^2 \bar{L}_{t+1}(z) \preceq (1+D/2)^{-1}+1/4\; .
\end{align*}
The function $\bar{L}_{t+1}(z)$ is therefore strongly convex and $(5/4+D/8)^{-1}$ well-conditioned. Furthermore, note the convexity of the constraint $\{z,\; \Lmat_t\minustransp z\in\Theta\}$ since $\Theta$ itself is convex. 

Let $\theta'_{t+1}$ be returned by the Projected Gradient Descent algorithm (see \cite[Algorithm 2]{hazan2016intro} for instance) ran for $T$ steps, where:
\begin{align*}
	T = (9/4+D/8)\log(\diam(\Theta_t)/\eps_t)\; .
\end{align*}
By \cref{prop:convex_optim_precision} this ensures that:
\begin{align*}
	\| \theta_{t+1} - \theta_{t+1}'\| \leq \eps_t\; ,
\end{align*}
which is enough to complete round $t$ of the \texttt{ECOLog} procedure. 
Because the gradients of $\bar{L}_{t+1}(\theta)$ only take $\bigo{d^2}$ operations to compute, the cost of running the Projected Gradient Descent algorithm for $T$ rounds is $\bigo{T(d^2+\proj(\Theta_t}))$. The quantity $\proj(\Theta_t)$ is the cost of projection the estimate on the set $\{\Lmat_t\minustransp z\in\Theta_t\}$. This constraint set is ellipsoidal since $\Theta_t$ is an ellipsoid (by assumption). Projecting on this set therefore boils down to solving a one-dimensional convex problem (see \cref{lemma:proj_ellipsoid}). Similarly, this program is solved to accuracy $\eps$ in $\bigo{d^2\log(1/\eps)}$ (it involves some matrix-vector multiplications and triangular inverse solving, hence the $d^2$ dependency). 
To finish the proof we are therefore left with evaluating the cost of computing the Cholesky factor $\Lmat_t$. This quantity can be maintained online and updated at cost $\bigo{d^2}$ thanks to the rank-one nature of $\Wmat_{t}$'s update (see for instance \citet[Section 6.5.4]{golub13}). 
\end{proof}

\subsubsection{Proof of \cref{prop:per_round_cost}}\label{app:proof_per_round_cost_no_more_naming_please}

\propperroundcost*

\begin{proof}
Recall that in \texttt{OFU-ECOLog} we have $\Theta_t\equiv \Theta$ where $D=1$ satisfies $D\geq \diam_\mcal{A}(\Theta)$ (see \cref{prop:diameter_warmup}) and $\eps_t=1/t$. Easy computations further show that $\diam(\Theta)\leq \poly(S)$; for instance, a crude bound yields that for all $\theta_1,\theta_2\in\Theta$
\begin{align*}
	\left\| \theta_1-\theta_2 \right\|^2 &\leq\lambda_\tau(\delta)^{-1} \left\| \theta_1-\theta_2 \right\|^2_{\Vmat_\tau} &(\Vmat_\tau\geq\lambda_\tau(\delta)\mbold{I}_d)\\
	&\leq 4 \lambda_\tau(\delta)^{-1} \beta_\tau(\delta) &(\theta_1,\theta_2\in\Theta) \\
	&= \poly(S) &\text{(see \cref{eq:def_lambda,eq:def_beta})}
\end{align*}
 \cref{prop:computational_cost} hence ensures that the cost running the \texttt{ECOLog} routine at round $t$ of \texttt{OFU-ECOLog} is at most $\const d^2\log(\poly(S)t)^2$. The optimistic planning mechanism requires performing $K$ matrix-vector products, which cost is $\bigo{Kd^2}$. This finishes the proof. 
\end{proof}

\begin{rem*}
The proof discards the cost of the warm-up; its only computational bottleneck is the computation of $\hat\theta_\tau$. This happens only once and boils down to the minimization of a well-conditioned (after preconditioning by $\Vmat_\tau$) convex function - which is therefore cheap, typically $\bigo{\tau\log(T)}$ where $\tau$ is the length of the warm-up. 
\end{rem*}

\subsubsection{Proof of \cref{prop:dilutedcomplexity}}

\propdilutedcomplexity*

\todo{Louis: here we have a slight overselling in the statement; missing some $d$ and $\log(T)$} 

\begin{proof}
The proof is essentially the same as for \cref{prop:per_round_cost}. The main difference is the value of $\diam_\mcal{A}(\Theta)$; it is now bounded by $\poly(S)$ (by using a similar argument that in \cref{app:proof_per_round_cost_no_more_naming_please} when we bounded $\diam(\Theta)$). As discussed in the main text there is however an additional cost inherited from the computations of $\hat\theta_t^\mcal{H}$. This requires minimizing a well-conditioned (after preconditioning by $\Vmat^{\mcal{H}}_t$) convex function which gradients are computed at $\bigo{d\vert \mcal{H}_t\vert}$ cost. \cref{lemma:bounding_tau} proves that $\vert \mcal{H}_t\vert\leq \kappa$; therefore the computational overhead is $\bigo{\kappa d\log(T)}$. Note that precisely because of \cref{lemma:bounding_tau}, it turns out that this extra-cost only needs to be paid at most $\approx \kappa$ times (and not at every round as suggested by \cref{prop:dilutedcomplexity}). 
\end{proof}

\subsection{Computational Costs of Other Approaches}\label{app:cost_others}
We briefly discuss the computational cost we announced in \cref{tab:regret_comparison} for \texttt{GLM-UCB} \cite{filippi2010parametric} and \texttt{OFULog-r} \cite{abeille2020instance}. Both require the computation of the MLE estimator:
\begin{align*}
	\hat\theta _{t+1} = \argmin_\theta \left\{ L_{t+1}(\theta) \defeq \sum_{s=1}^t \ell_{s+1}(\theta) + \lambda\|\theta\|^2\right\}\; .
\end{align*}
An efficient way to solve $\hat\theta_{t+1}$ to $\eps$ accuracy (typically with $\eps=1/T$ to preserve regret guarantees) is to run a gradient descent (GD) algorithm with $\Vmat_t$-preconditioning. This step is important as in all generality $L_{t+1}$ can be $1/t$ well-conditioned; running GD directly on $L_{t+1}$ will therefore require $\bigo{t\log(1/\eps}$ to reach $\eps$-accuracy. Preconditioning allows to reduce this cost to $\bigo{\log(1/\eps)}$. The cost of computing the gradient of $L_{t+1}$ (and its pre-conditioned version) is still high, typically $\Omega(t)$ (more precisely, $\Omega(d^2t)$ for the preconditioned version which involves matrix-vector multiplication). Overall, the cost of computing $\hat\theta_{t+1}$ to $\eps$ accuracy is therefore $\bigo{d^2t\log(1/\eps)}$. 
  \todo{This means that we could put $\bigotilde{}$ for these two approach in the table ..}

For \texttt{GLM-UCB} a $\bigo{d^2K}$ additional cost is to be added to account for the optimistic planning. Things are worse for \texttt{OFULog-r} as at round $t$ and \emph{for every arm} $a\in\mcal{A}$ it needs to solve a convex program of the form:
\begin{align*}
	\max_\theta \left\{ a\transp\theta \text{ s.t } L_{t+1}(\theta) \leq \gamma_t(\delta) \right\}\; .
\end{align*}
Projecting on the set $\{ L_{t+1}(\theta) \leq \gamma_t(\delta)\}$ is as costly as computating of $\hat\theta_{t+1}$, hence the additional $\bigotilde{Kd^2T}$ computational cost.

%% file: appendix/aux.tex
\newpage
\section{AUXILIARY RESULTS}\label{app:aux}

The following version of the Elliptical Potential lemma (see, \emph{e.g}, \cite[Lemma 11]{abbasi2011improved}) is a direct consequence of \cite[Lemma 15]{faury2020improved} along with the determinant-trace inequality (see \cref{lemma:determinant_trace_inequality}).

\begin{lemma}[Elliptical potential]\label{lemma:ellipticalpotential}
    Let $\lambda\geq 1$ and $\{x_s\}_{s=1}^\infty$ a sequence in $\mbb{R}^d$ such that $\|{x_s}\|\leq X$ for all $s\in\mbb{N}$. Further, define $\mbold{V}_t \defeq \sum_{s=1}^{t} x_sx_s\transp+\lambda\mbold{I}_d$. Then:
    $$
        \sum_{t=1}^{T} \left\lVert x_{t}\right\rVert_{\mbold{V}_{t-1}^{-1}}^2 \leq 2d(1+X^2)\log\left(1 + \frac{TX^2}{d\lambda}\right)
   $$
\end{lemma}

The following is extracted from \cite[Lemma 10]{abbasi2011improved}.
\begin{lemma}[Determinant-Trace inequality]\label{lemma:determinant_trace_inequality}
     Let $\{x_s\}_{s=1}^\infty$ a sequence in $\mbb{R}^d$ such that $\|{x_s}\|\leq X$ for all $s\in\mbb{N}$, and  let $\lambda$ be a non-negative scalar. For $t\geq 1$ define $\mbold{V}_t \defeq \sum_{s=1}^{t} x_sx_s\transp+\lambda\mbold{I}_d$. The following inequality holds:
     \begin{align*}
         \det(\mbold{V}_{t}) \leq \left(\lambda+tX^2/d\right)^d
     \end{align*}
\end{lemma}

The following statements are standard results from the convex optimization literature. 
\begin{lemma}[Section 4.2.3 of \cite{boyd_vandenberghe_2004}] \label{fact:kkt}
	Let $f:\mbb{R}^d\to\mbb{R}$ a differentiable and convex function and $\mcal{C}\subset\mbb{R}^d$ a convex set. Further, denote:
	\begin{align*}
		x_0 := \argmin_{x\in\mcal{C}} f(x)\; .
	\end{align*}
	Then for \emph{any} $y\in\mcal{C}$:
	\begin{align*}
		\nabla f(x_0)\transp(y-x) \geq 0 \; .
	\end{align*}
\end{lemma}

\begin{lemma}\label{prop:convex_optim_precision}
	Let $f:\mbb{R}^d\to\mbb{R}$ a twice differentiable and strongly convex function such that for all $x\in\mbb{R}^d$:
	\begin{align*}
		0\preceq \alpha \mbold{I}_d \preceq \nabla^2f(x) \preceq \beta\mbold{I}_d\;. 
	\end{align*}
	Let $\mcal{C}\subset\mbb{R}^d$ a convex set, $x_0 = \argmin_\mcal{C}f(x)$ and $\gamma=\alpha/\beta$. Let $x_{T+1}$ be the estimator returned by the projected gradient descent algorithm (Algorithm 2 in \cite{hazan2016intro}) with step-size $1/\beta$ run for $T$ rounds. For $\varepsilon>0$ setting:
	\begin{align*}
		T = (1+\gamma^{-1}) \log\left(\diam(C)/\varepsilon\right)\; ,
	\end{align*}
	ensures that $\| x_{T+1} - x_0\| \leq \varepsilon$. 
\end{lemma}

\begin{proof}
The proof is standard in the convex optimization literature. We remind it briefly for completeness.
\begin{align*}
	f(x_{T+1}) &\geq f(x_0) + \nabla f(x_0)\transp(x_{T+1}-x_0)+  \frac{\alpha}{2}\| x_{T+1}-x_0\|^2 \\
	&\geq  f(x_0) +  \frac{\alpha}{2}\| x_{T+1}-x_0\|^2 &(\text{\cref{fact:kkt}}) 
\end{align*}
Furthermore by convexity:
\begin{align*}
f(x_{T+1}) &\leq f(x_T) + \nabla f(x_T)\transp(x_{T+1}-x_T) + \frac{\beta}{2}\| x_{T+1}- x_T\|^2 \\
&\leq f(x_T) + \nabla f(x_T)\transp(x_0-x_T) - \frac{\beta}{2}\| x_{T+1}- x_{0}\|^2 + \frac{\beta}{2}\| x_{T}- x_0\|^2\\
&\leq f(x_0) -  \frac{\beta}{2}\| x_0- x_{T+1}\|^2+  \frac{\beta-\alpha}{2}\| x_{T}- x_0\|^2\; .
\end{align*}
The second to last inequality uses the definition of $x_{T+1}$ (given by the projected gradient descent algorithm). Plugging everything together yields:
\begin{align*}
	\| x_{T+1}-x_0\|^2  &\leq \frac{\beta-\alpha}{\beta+\alpha} \| x_{T}- x_0\|^2\\
	&\leq \left(\frac{\beta-\alpha}{\beta+\alpha}\right)^T \| x_{1}- x_0\|^2\\
	&\leq \left(1-\frac{2\alpha}{\beta+\alpha}\right)^T\diam(\mcal{C})^2\\
	&\leq \exp(-2T\alpha/(\beta+\alpha))\diam(\mcal{C})^2\; .
\end{align*}
Solving for $\| x_{T+1}-x_0\|^2 \leq \eps^2$ yields the announced result.
\end{proof}

\begin{lemma}[Ellipsoidal Projection]\label{lemma:proj_ellipsoid} Let $x\in\mbb{R}^d$ and $\mbold{A}\in\mbb{R}^{d\times d}$ a p.s.d matrix. Let $y$ be the projection of $x$ onto the set $\{z, \|z\|^2_{\mbold{A}}/2\leq 1\}$. Then $ y =   (\mbold{I}_d+\lambda_\star\mbold{A}^{-1})^{-1} x$ where $\lambda_\star$ is the solution of the following one-dimensional strongly concave program:
\begin{align*}
	\lambda_\star =  \argmax_{\lambda\geq 0} - 2\lambda - x\transp \mbold{A}^{1/2}(\lambda\mbold{I}_d+\mbold{A})^{-1} \mbold{A}^{1/2} x\; .
\end{align*}
\end{lemma}
\begin{proof}
	By definition of the projection onto a convex set:
	\begin{align*}
		y \defeq \argmin_{\frac{1}{2}\|z\|^2_{\mbold{A}^{-1}}\leq 1}\left\{ f(z) \defeq \frac{1}{2} \| x-z\|^2\right\}\; 
	\end{align*}
	Introducing the Lagrangian $L(z,\lambda)\defeq  \| x-z\|^2/2 + \lambda(\|z\|^2_{\mbold{A}^{-1}}/2-1)$ and by strong duality:
	\begin{align*}
		f(y) &= \min_z \max_{\lambda\geq 0} L(z,\lambda) \\
		&= \max_{\lambda\geq 0} \min_z L(z,\lambda) \; .
	\end{align*}
Denoting $z(\lambda) = \argmin_z L(z,\lambda)$, direct computation yields that:
\begin{align*}
	z(\lambda) =  (\mbold{I}_d+\lambda\mbold{A}^{-1})^{-1} x\; .
\end{align*}
Replacing into the dual problem, one obtains $y=z(\lambda_\star)$ where $\lambda_\star$ solves the program:
\begin{align*}
	\lambda_\star = \argmax_{\lambda\geq 0} - \lambda - x\transp \mbold{A}^{1/2}(\lambda\mbold{I}_d+\mbold{A})^{-1} \mbold{A}^{1/2} x/2 \; .
\end{align*}
\end{proof}

%% file: appendix/exps.tex
\newpage
\section{ADDITIONAL EXPERIMENTS} \label{app:exps}

The results reported in \cref{fig:add_exps} complements \cref{fig:comp} from the main text, for \logb{} instances of higher dimension and varying values of $\kappa$. As promised by the regret bounds, the improvement brought by \texttt{ada-OFU-ECOLog} over its statistically sub-optimal predecessors increases as $\kappa$ grows (\emph{i.e} as the reward signal gets more non-linear). We did not evaluate the performances of \texttt{OFULog-r} in this setting - it is unfortunately too computationally demanding to complete in reasonable time.

\begin{figure*}[!th]
\begin{subfigure}{0.33\textwidth}
	\includegraphics[width=\textwidth]{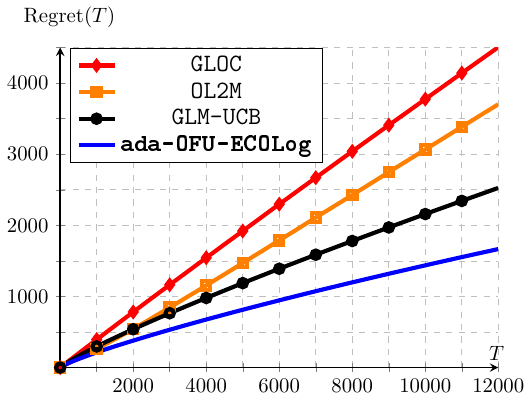}
	\caption*{$d=10$, $\kappa=50$}
\end{subfigure}
\begin{subfigure}{0.33\textwidth}
	\includegraphics[width=\textwidth]{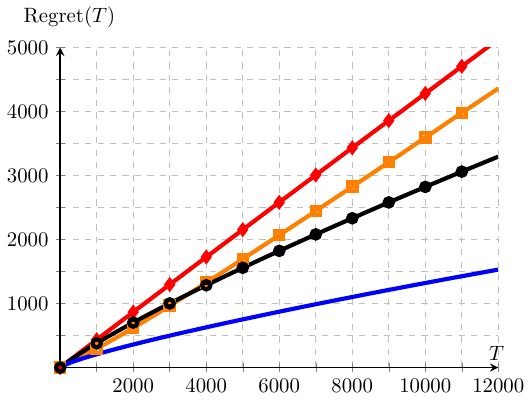}
	\caption*{$d=10$, $\kappa=150$}
\end{subfigure}
\begin{subfigure}{0.33\textwidth}
	\includegraphics[width=\textwidth]{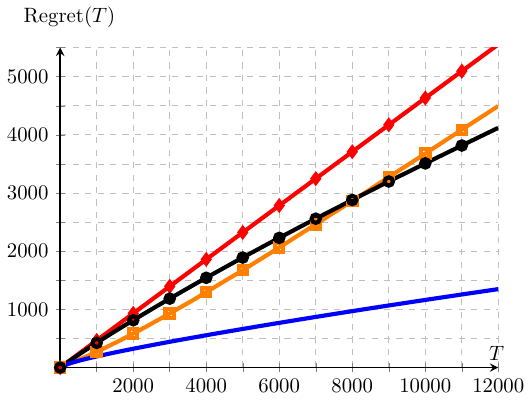}
	\caption*{$d=10$, $\kappa=400$.}
\end{subfigure}
\caption{Numerical simulations on \logb{} problems of dimensions $d=10$ and varying value of $\kappa$. Regret curves are averaged over 100 independent trajectories, for fixed arm-sets of cardinality 200.}\label{fig:add_exps}
\end{figure*}